\documentclass[12pt]{article}
\usepackage{tikz}
\usepackage{verbatim}
\usepackage{natbib}
\usepackage{etoolbox}

\usetikzlibrary{shapes,arrows,chains}
\usetikzlibrary[calc]

\usepackage{xspace}
\usepackage{verbatim}
\usepackage[margin=0.75in]{geometry}
\usepackage{graphicx}
\usepackage{tabularx}
\usepackage{mleftright}
\usepackage{amsmath}
\usepackage{bbm}
\usepackage[showonlyrefs]{mathtools}
\usepackage{mathstyle}
\usepackage{breqn}
\usepackage{empheq}
\usepackage{amstext,amssymb,amsfonts}
\usepackage{fullpage}
\usepackage{nicefrac}
\usepackage{boxedminipage}
\usepackage{ifdraft}
\usepackage{bm}

\usepackage{algorithm2e}
\usepackage{caption}
\usepackage{subcaption}

\usepackage[breaklinks=true]{hyperref}
\hypersetup{
	colorlinks=true,
	linkcolor=blue,
	citecolor=blue,
	urlcolor=blue
}

\usepackage{amsthm}
\usepackage{thmtools}
\usepackage{thm-restate}

\declaretheorem[within=section]{theorem}
\declaretheorem[sibling=theorem]{corollary}

\declaretheorem[sibling=theorem]{lemma}

\declaretheorem[sibling=theorem]{definition}

\declaretheorem[sibling=theorem]{remark}
\declaretheorem[sibling=theorem]{question}

\newcommand{\poly}{\mathrm{poly}}

\newcommand{\dist}{\mathrm{dist}}

\renewcommand{\vec}[1]{{\bf{#1}}}

\renewcommand{\leq}{\leqslant}
\renewcommand{\geq}{\geqslant}
\renewcommand{\ge}{\geqslant}
\renewcommand{\le}{\leqslant}
\renewcommand{\epsilon}{\varepsilon}
\newcommand{\eps}{\epsilon}

\newcommand{\N}{\mathbb{N}}
\newcommand{\F}{\mathbb{F}}

\newcommand{\cA}{\mathcal A}

\newcommand{\cD}{\mathcal D}

\newcommand{\cF}{\mathcal F}

\newcommand{\cH}{\mathcal H}
\newcommand{\cI}{\mathcal I}

\newcommand{\cO}{\mathcal O}

\newcommand{\cX}{\mathcal X}
\newcommand{\cY}{\mathcal Y}
\newcommand{\cZ}{\mathcal Z}

\newcommand{\Psymb}{{\bf Pr}}

\DeclareMathOperator*{\ProbOp}{\Psymb}

\renewcommand{\Pr}{\ProbOp}

\newcommand{\Ind}{\mathbb I}

\title{Private PAC Learning May be Harder than Online Learning}
\usepackage{times}

\usepackage{authblk}
\author[1]{Mark Bun}
\author[2]{Aloni Cohen}
\author[1]{Rathin Desai}
\affil[1]{Boston University}
\affil[2]{University of Chicago}

\date{}

\begin{document}
\maketitle

\SetKwProg{Init}{Initialize}{}{}
\SetKwComment{Comment}{/* }{ */}
\RestyleAlgo{ruled}

\newcommand{\algstyle}[1]{\mathsf{#1}}
\newcommand{\inst}{X}
\newcommand{\expt}{\mathbb{E}}
\newcommand{\class}{\cF}
\newcommand{\concept}{f}

\newcommand{\floor}[1]{\lfloor{#1}\rfloor}

\newcommand{\lkg}{\algstyle{leak}}
\newcommand{\ciph}{c}
\newcommand{\plain}{m}
\newcommand{\st}{\algstyle{st}}
\newcommand{\tfld}{\algstyle{tfld}}
\newcommand{\ciphvec}{\vec{\ciph}}
\newcommand{\dst}{\algstyle{dist}}

\newcommand{\fe}{\algstyle{1FE}}
\newcommand{\fes}{\algstyle{1FE.S}}
\newcommand{\fekg}{\algstyle{1FE.KG}}
\newcommand{\fee}{\algstyle{1FE.E}}
\newcommand{\fed}{\algstyle{1FE.D}}

\newcommand{\FRE}{\algstyle{FRE}}
\newcommand{\ORE}{\algstyle{ORE}}
\newcommand{\Gen}{\algstyle{Gen}}
\newcommand{\Enc}{\algstyle{Enc}}
\newcommand{\Dec}{\algstyle{Dec}}
\newcommand{\Comp}{\algstyle{Comp}}
\newcommand{\Eval}{\algstyle{Eval}}
\newcommand{\params}{\algstyle{params}}


\newcommand{\ind}{\mathbf{1}}
\newcommand{\secparam}{\lambda}
\newcommand{\EncThr}{\algstyle{LEncThr}}
\newcommand{\fld}{\algstyle{fld}}
\newcommand{\Adv}{\algstyle{Adv}}
\newcommand{\Chl}{\algstyle{Chl}}


\newcommand{\ZKsetup}{\algstyle{Setup}}
\newcommand{\ZKprove}{\algstyle{Prove}}
\newcommand{\ZKver}{\algstyle{Ver}}
\newcommand{\ZKcom}{\algstyle{Com}}
\newcommand{\msg}{x}
\newcommand{\msk}{\algstyle{msk}}
\newcommand{\ct}{\algstyle{ct}}
\newcommand{\PP}{PP}
\newcommand{\crs}{\algstyle{crs}}

\newcommand{\sk}{\mathsf{sk}}
\newcommand{\inner}[2]{\langle #1,#2\rangle}

\newcommand{\mon}{\algstyle{mon}}
\newcommand{\monEnc}{\algstyle{\mon.Enc}}
\newcommand{\monDec}{\algstyle{\mon.Dec}}

\newcommand{\io}{\algstyle{iO}}

\newcommand{\mife}{\algstyle{2FE}}
\newcommand{\mikeygen}{\algstyle{2FE.KeyGen}}
\newcommand{\midec}{\algstyle{2FE.Dec}}
\newcommand{\mienc}{\algstyle{2FE.Enc}}
\newcommand{\misetup}{\algstyle{2FE.Setup}}
\newcommand{\enckey}{\algstyle{EK}}
\newcommand{\cipher}{\algstyle{CT}}
\newcommand{\crsgen}{\algstyle{CRSGen}}
\newcommand{\pke}{\algstyle{PKE}}
\newcommand{\pkenc}{\algstyle{PKE.Enc}}
\newcommand{\pkesetup}{\algstyle{PKE.Setup}}
\newcommand{\pkdec}{\algstyle{PKE.Dec}}
\newcommand{\pk}{\algstyle{pk}}

\newcommand{\tfe}{\algstyle{3FE}}
\newcommand{\tkeygen}{\algstyle{3FE.KeyGen}}
\newcommand{\tdec}{\algstyle{3FE.Dec}}
\newcommand{\tenc}{\algstyle{3FE.Enc}}
\newcommand{\tsetup}{\algstyle{3FE.Setup}}

\newcommand{\prf}{\algstyle{PRF}}
\newcommand{\prfgen}{\algstyle{PRF.Gen}}
\newcommand{\prfeval}{\algstyle{PRF.Eval}}

\newcommand{\nizk}{\algstyle{NIZK}}
\newcommand{\negl}{\algstyle{negl}}

\begin{abstract}
We continue the study of the computational complexity of differentially private PAC learning and how it is situated within the foundations of machine learning. A recent line of work uncovered a qualitative equivalence between the private PAC model and Littlestone's mistake-bounded model of online learning, in particular, showing that any concept class of Littlestone dimension $d$ can be privately PAC learned using $\poly(d)$ samples. This raises the natural question of whether there might be a generic conversion from online learners to private PAC learners that also preserves computational efficiency.

We give a negative answer to this question under reasonable cryptographic assumptions (roughly, those from which it is possible to build indistinguishability obfuscation for all circuits). We exhibit a concept class that admits an online learner running in polynomial time with a polynomial mistake bound, but for which there is no computationally-efficient differentially private PAC learner. Our construction and analysis strengthens and generalizes that of Bun and Zhandry (TCC 2016-A), who established such a separation between private and non-private PAC learner.
\end{abstract}

\tableofcontents


\section{Introduction}

Differential privacy~\citep{DworkMNS06} is a formal guarantee of individual-level privacy for the analysis of statistical datasets. Algorithmic research on differential privacy has revealed it to be a central concept to theoretical computer science and machine learning, supplementing the original motivation with 
deep connections to diverse topics including mechanism design~\citep{McSherryT07, NissimST12}, cryptography~\citep{BeimelHMO18}, quantum computing~\citep{AaronsonR19}, generalization in the face of adaptive data analysis~\citep{DworkFHPRR15, HardtU14}, and replicability in learning~\citep{BunGHILPSS23}. 

To investigate the connections between privacy and machine learning in a simple and abstract setting,~\cite{KasiviswanathanLNRS11} introduced the \emph{differentially private PAC model} for binary classification.  Numerous papers~\citep{BeimelBKN14,BunNSV15,FeldmanX15,BeimelNS16,BunZ16,Beimel19Pure,AlonLMM19,KaplanLMNS19,Bun20,SadigurschiS21} have since explored the capabilities and limitations (both statistical and computational) of algorithms in this model. A major motivating question in this area is:

\begin{question} \label{q:general}
When do sample-efficient private PAC learners exist, and 
when can they be made computationally efficient?
\end{question}

Early work~\citep{BlumDMN05, KasiviswanathanLNRS11} gave us some important partial answers. On the statistical side, they showed that every finite concept class $\class$ can be privately learned with $O(\log |\class|)$ samples, albeit by an algorithm taking exponential time in general. Computationally, they showed that two ``natural'' paradigms for polynomial-time non-private PAC learning have differentially private  analogs: learners in Kearns' statistical query (SQ) model~\citep{Kearns98} and the Gaussian elimination-based learner for parities. Much of the subsequent work on private PAC learning has focused on improving the sample- and computational-efficiency of algorithms for fundamental concept classes, including points, thresholds, conjunctions, halfspaces, and geometric concepts. Meanwhile,~\cite{BunZ16} gave an example of a concept class that has a polynomial-time non-private PAC learner, but no computationally-efficient private PAC learner (under strong, but reasonable cryptographic assumptions).

While this work has led to the development of important algorithmic tools and to fascinating connections to other areas of theoretical computer science, a general answer to Question~\ref{q:general} continues to elude us. Some tantalizing progress was made in a recent line of work connecting private PAC learning to the completely different model of mistake-bounded online learning. This connection is summarized as follows.

\begin{theorem}[\citet{AlonBLMM22, GhaziGKM21}] \label{thm:private-online-equiv}
	Let $\class$ be a concept class with \emph{Littlestone dimension} $d = L(\class)$. Then $\tilde{O}(d^6)$ samples are sufficient to privately learn $\class$ and $\Omega(\log^* d)$ samples are necessary.
\end{theorem}

Here, the Littlestone dimension of a class $\class$ measures the best possible mistake bound in Littlestone's model of online learning. Thus, at least qualitatively, online learnability characterizes private PAC learnability. In particular, the ``are sufficient'' direction of Theorem~\ref{thm:private-online-equiv} can be viewed as an elaborate online-to-batch conversion, transforming an online learner into a private PAC learner with only a polynomial blowup from mistake bound to private sample cost.  Unfortunately, while it is statistically efficient, the algorithm achieving this~\citep{GhaziGKM21} does not preserve computational efficiency in general. Among other steps, it entails computing $L(\class')$ for various subclasses derived from $\class$, which is believed to be computationally hard in general~\citep{Schaefer99, FrancesL98, ManurangsiR17, Manurangsi23}.


One might nonetheless hope for a different transformation that maintains both statistical and computational efficiency. Indeed, when one's goal is to convert an online learner to a \emph{non-private} PAC learner, the transformation is simple and clearly efficient -- just present random examples to the online learner and output its eventual state as a classifier~\citep{Littlestone89} (see also the variant due to~\cite{KearnsLPV87, Angluin88}, which is a standard topic in graduate classes on learning theory). Moreover, many key techniques from online learning have found differentially private analogs incurring minimal overhead, including follow-the-regularized-leader~\citep{AgarwalS17} and learning from experts~\citep{AsiFKT23}.

Our main result shows that a generic transformation is unlikely to exist.

\begin{theorem}[Informal] \label{thm:online-not-private}
	Under (strong, but reasonable) cryptographic assumptions, there is a concept class that is online learnable by a polynomial-time algorithm with a polynomial mistake bound, but not privately PAC learnable in polynomial-time.
\end{theorem}

A formal description of this result appears as Theorem~\ref{thm:main} in Section~\ref{sec:crypto}, along with further discussion of the cryptographic assumptions. Roughly, our assumptions include functional encryption for all poly-size circuits (an assumption comparable to indistinguishability obfuscation, and recently shown to be obtainable from reasonable assumptions~\citep{JainLS21}), a circuit lower bound for deterministic exponential time, and perfectly sound non-interactive zero knowledge proofs.

Note that owing to the existence of efficient non-private online-to-batch conversions mentioned above, this result strengthens the separation between private and non-private learning from~\cite{BunZ16}. It also complements a pair of papers studying the possibility of a computationally efficient transformation in the opposite direction. Namely,~\cite{GonenHM19} gave conditions under which pure private learners can be efficiently converted into online learners, while~\cite{Bun20} gave a counterexample of a class that is efficiently privately PAC learnable, but not efficiently online learnable. Our result answers an open question from~\cite{Bun20} and, together with that result, shows that polynomial-time online learnability and polynomial-time private PAC learnability are technically incomparable.

\subsection{Techniques}

Our construction of a concept class that separates online learning from private PAC learning builds on the construction from~\cite{BunZ16}, so let us briefly review it here. The starting point of their construction was the concept class of one-dimensional threshold functions $\algstyle{Thr}$ over a domain of the form $[N] = \{1, \dots, N\}$. Each function $f_t \in \algstyle{Thr}$ is itself parameterized by a value $t \in [N]$, and takes the value $f_t(x) = 1$ if $x < t$, and $f_t(x) = 0$ otherwise. Threshold functions are easy to PAC learn non-privately. Given a sampled dataset $((x_i, y_i))_{i = 1}^n$, a non-private algorithm can simply output $f_{x_i}$ where $x_i$ is the largest value for which the label $y_i = 1$. A standard concentration argument shows that this generalizes to the underlying population from which the sample is drawn as long as $n$ is larger that some constant that is independent of the domain size $N$.

On the other hand, this algorithm badly fails to be differentially private, as it exposes the sample $x_i$. In fact, every differentially private learner for this class requires $\Omega(\log^* N)$ samples~\citep{BunNSV15, AlonBLMM22}. However, this lower bound is too small to give a computational separation, so the idea in~\cite{BunZ16} was to use cryptography to preserve the problem's non-private learnability, while making it much harder to achieve differential privacy. Specifically, they defined a class $\algstyle{EncThr}$ by first encrypting each example $x_i$ under an \emph{order-revealing encryption} scheme. Such a scheme allows for ciphertexts to be compared in a manner consistent with the underlying plaintexts, but for nothing else to be revealed besides their order. The ability to make these comparisons is enough for the simple non-private ``largest positive example'' algorithm to go through. Meanwhile, security of the order-revealing encryption scheme intuitively guarantees that comparisons to the specific ciphertexts appearing in the sample are \emph{all} that efficient learners can do, and hence they cannot be differentailly private.

Turning now to our goal of separating online from private PAC learning, we observe that while the class $\algstyle{Thr}$ is efficiently online learnable with mistake bound $\log N$ via binary search (see Section~\ref{sec:online-threshold}), the encrypted class $\algstyle{EncThr}$ is \emph{not}. Intuitively, order-revealing encryption does not reveal enough information to enable a learner to make efficient use of mistakes, as in binary search. More precisely, suppose the target concept is the (encrypted version) of the middle threshold $f_{N/2}$. Consider an adversary who selects examples by randomly choosing either the smallest positive example not presented so far, or the largest negative example. A computationally-bounded adversary who can only compare these examples to those seen so far cannot distinguish between these cases, and is hence liable to make super-polynomially (depending on the security of the underlying ORE) many mistakes.

To obtain our main result, we modify the class $\algstyle{EncThr}$ to make it efficiently online learnable, while keeping it hard enough to carry out a lower bound against differentially private algorithms. Achieving both goals simultaneously turns out to be a delicate task, and it helps to think about each in more abstract terms. To this end, let $\algstyle{L}$ be a \emph{function-revealing encryption} scheme, which (generalizing ORE) enables the revelation (``leakage'') of specific structured relationships between plaintexts, but nothing else. Correspondingly, let $\algstyle{LEncThr}$ be the class of one-dimensional thresholds with examples encrypted under $\algstyle{L}$. Given the inadequacy of ORE for efficient online learning, we will think of $\algstyle{L}$ as revealing not only the order of plaintexts, but some limited information about the distances between plaintexts as well.

First, we identify sufficient conditions on $\algstyle{L}$ to enable the construction of an efficient online learner. Inspired by binary search, we'd like to reveal enough distance information so that every mistake made by an online learner can rule out a constant fraction of the remaining space of consistent concepts. Thinking of distances on a logarithmic scale, we articulate this condition as a ``bisection property'' of the leakage function (Definition~\ref{def:bisection}) and analyze our analog of binary search in Section~\ref{sec:online-learner}.

Second, we identify sufficient conditions on $\algstyle{L}$ to enable a lower bound against differentially private PAC learners. To do so, we simplify the lower bound argument from~\cite{BunZ16}, in particular, bypassing their intermediate abstraction of an ``example re-identification scheme'' and directly showing how to use an accurate, efficient, differentially private learner for $\algstyle{EncThr}$ to construct an adversary violating the security of the underlying $\algstyle{ORE}$ scheme. This simplified argument goes roughly as follows. Consider running a PAC learner on $n$ uniformly random encryptions labeled by the middle threshold $f_{N/2}$. Accuracy of the learner, together with an averaging argument, implies that for some index $i$, it can distinguish random encryptions of messages from $[x_{i-1}, x_i)$ from random encryptions from $[x_i, x_{i+1})$ with advantage $\Omega(1/n)$. Now, differential privacy implies that this noticeable distinguishing advantage remains even when the learner is \emph{not} given example $x_i$, violating the security of the ORE scheme.

This simplified argument makes it easy to use group differential privacy to reason about what happens when not just one, but a small number of examples are removed. In particular, an inverse polynomial distinguishing advantage remains even if we withhold $O(\log n)$ points from the learner. This extra flexibility turns out to be critical in helping us construct pairs of challenge messages in $\algstyle{L}$ security games, which not only need to have the same relative order, but respect the stronger constraints imposed by $\algstyle{L}$-leakage. We describe the precise condition we need as ``log-invariance'' (Definition~\ref{def:swaplog}) and show in Section~\ref{sec:hardness} that any $\algstyle{L}$ with this condition gives rise to a private PAC learning lower bound.

Our final task is to exhibit a function-revealing encryption scheme $\algstyle{L}$ that actually has both the bisection and log-invariance properties. Identifying a leakage function that works turns out to be quite tricky. Our starting point is to reveal the floor of the logarithm of the distance between plaintexts, but this gives too much information for a lower bound to hold. Instead, we reduce this to just comparison information between floor log distances. That is, for any triple of plaintexts $m_0 \le m_1 \le m_2$, we leak whether the floor log distance between $m_0$ and $m_1$ is less than that between $m_1$ and $m_2$.

\section{Preliminaries}

\subsection{PAC Learning}

For each $d \in \N$, let $\inst_d$ be an instance space (such as $\{0, 1\}^d$), where the parameter $d$ represents the size of the elements in $\inst_d$. Let $\class_d$ be a set of boolean functions $\{\concept : \inst_d \to \{0, 1\}\}$. The sequence $(\inst_1, \class_1), (\inst_2, \class_2), \dots$ represents an infinite sequence of learning problems defined over instance spaces of increasing dimension. We will generally suppress the parameter $d$, and refer to the problem of learning $\class$ as the problem of learning $\class_d$ for every $d$.

A learner $L$ is given examples sampled from an unknown probability distribution $\cD$ over $\inst$, where the examples are labeled according to an unknown {\em target concept} $\concept\in \class$. The learner must select a hypothesis $h$ from a hypothesis class $\cH$ that approximates the target concept with respect to the distribution $\cD$. We now define the notion of PAC (``Probably Approximately Correct'') learning formally.

\begin{definition}
The generalization error of a hypothesis $h:\inst\rightarrow\{0,1\}$ (with respect to a target concept $\concept$ and distribution $\cD$) is defined by
$error_{\cD}(\concept,h)=\Pr_{x \sim \cD}[h(x)\neq \concept(x)].$
If $error_{\cD}(\concept,h)\leq\alpha$ we say that $h$ is an {$\alpha$-good} hypothesis for $\concept$ on $\cD$.
\end{definition}

\begin{definition}[PAC Learning, \cite{Valiant84}]\label{def:PAC}
Let $\cH$ be a class of boolean functions over $X$. An algorithm $L : (\inst \times \{0, 1\})^n \to \cH$ is an {\em $(\alpha,\beta)$-accurate PAC learner} for the concept class $\class$ using hypothesis class $\cH$ with sample complexity $n$ if for all target concepts $\concept \in \class$ and all distributions $\cD$ on $X$, given as input $n$ samples $S =((x_i, \concept(x_i)),\ldots,(x_n, \concept(x_n)))$, where each $x_i$ is drawn i.i.d.\ from $\cD$, algorithm $L$ outputs a hypothesis $h\in \cH$ satisfying
$\Pr[error_{\cD}(\concept,h)  \leq \alpha] \geq 1-\beta$. The probability here is taken over the random choice of the examples in $S$ and the coin tosses of the learner $L$.
\end{definition}

We are primarily interested in computationally efficient PAC learners, defined as follows. 
\begin{definition}[Efficient PAC Learning]
A PAC learner $L$ for concept class $\class$ is \emph{efficient} if it runs in time polynomial in the size parameter $d$, the representation size of the target concept $\concept$, and the accuracy parameters $1/\alpha$ and $1/\beta$.
\end{definition}

Note that a necessary (but not sufficient) condition for $L$ to be efficient is that its sample complexity $n$ is polynomial in the learning parameters.

\subsection{Differential Privacy}

We now define differential privacy and the differentially private PAC model.

\begin{definition}($k$-neighboring datasets)
 Let $S, S' \in Z^n$ for some data domain $Z$. We say that $S$ and $S'$ are $k$-neighboring datasets if they differ in exactly $k$  entries. If $k = 1$, we simply say they are neighboring.
\end{definition}

\begin{definition}[Differential Privacy, \cite{DworkMNS06, DworkKMMN06}]
An algorithm $M : Z^n \to R$ is  $(\eps, \delta)$-\emph{differentially private} if for all sets $T \subseteq \cH$, and neighboring datasets $S, S' \in Z^n$,
\[\Pr[M(S) \in T] \le e^{\eps}\Pr[M(S') \in T] + \delta.\]
\end{definition}

Specializing this definition to the case where $Z = X \times \{0, 1\}$ and $R = \cH$ is a class of hypotheses, we obtain the differentially private PAC model of~\cite{KasiviswanathanLNRS11}.

We now state some simple tools for designing and analyzing differentially private algorithms.

\begin{lemma}[Basic Composition~\cite{DworkMNS06, DworkL09}]
Let $M_1:Z^n\to R_1$ be $(\epsilon,\delta)$-differentially private. Let $M_2:Z^n\times R_1\to R_2$ be $(\epsilon,\delta)$-differentially private for every fixed value of its second argument. Then the composed algorithm $M: Z^n \to R_2$ defined by $M(S)=M_2(S,M_1(S))$ is $(\epsilon_1+\epsilon_2,\delta_1+\delta_2)$-differentially private.
\end{lemma}

\begin{lemma}[Group Privacy]\label{lemm:group}
Let $M:Z^n\to R$ be $(\epsilon,\delta)$-differentially private. Let $S$ and $S'$ be $k$-neighboring databases. Then for all sets of outcomes $T$, 
$$\Pr[M(S)\in T]\leq e^{k\epsilon}\cdot\Pr[M(S)\in T]+\frac{e^{k\epsilon}-1}{e^{\epsilon}-1}\cdot \delta.$$
\end{lemma}

\begin{lemma}[Post-Processing]\label{lemm:post-processing}
Let $M:\inst^n\to R$ be $(\epsilon,\delta)$-differentially private and let $f:R\to R'$ be an arbitrary randomized function. Then $f\circ M:\inst^n\to R'$ is $(\epsilon,\delta)$-differentially private.
\end{lemma}


\subsection{Online Learning, Halving, and Thresholds} \label{sec:online-threshold}
We review Littlestone's model of mistake-bounded learning~\citep{Littlestone87online}. It is defined as a two-player game between a learner and an adversary. Let $\class$ be a concept class. Prior to the start of the game, the adversary fixes a concept $\concept\in \class$. Let $\vert \concept\vert$ represent the description size of the concept and $d$ be the dimension of the instance space. The learning proceeds in rounds. In each round $i$,
\begin{enumerate}
    \item The adversary selects an $x_i\in \{0,1\}^d$ and reveals it to the learner.
    \item The learner predicts a label $\hat{y}_i\in \{0,1\}$.
    \item The adversary reveals the correct label $y_i= \concept(x_i)$.
\end{enumerate}

A learner makes a mistake every time $\hat{y}_i\neq \concept(x_i)$. The goal of the learning algorithm is to minimize the number of mistakes it makes in the game. A learning algorithm learns $\concept \in \class$ with mistake bound $M$ if for every target concept and adversary strategy, the total number of mistakes that the learner makes is at most $M$. We say that an online learner efficiently learns $\class$ if for every $\concept\in \class$ it has a mistake bound of $\poly(d,|\concept|)$ and runs in time $\poly(d,|\concept|)$ in every round.

A basic algorithm in this setting is the halving algorithm. 
Halving guarantees a mistake bound of $\log(|\class|)$ and can be made computationally efficient in certain structured cases. One such case is for the simple class of thresholds, which we describe below and study the halving algorithm for.

On data domain $X_d = [2^d]$, the class of thresholds $\algstyle{Thr}_d = \{f_t : X_d \to \{0, 1\}\}$ over domain $X_d$ is defined as follows. For each $t \in [2^d]$, define
\[f_{t}(x) = \begin{cases}
1 & \text{ if } x < t \\
0 & \text{ otherwise.}
\end{cases}\]

The halving algorithm for learning thresholds is described in Algorithm~\ref{alg:halving} below.

\SetKwComment{Comment}{/* }{ */}

\begin{algorithm}
\SetKwInOut{Input}{Input}
\caption{Halving}\label{alg:halving}
\textbf{Initialize:} $t \gets 2^{d-1},t_+\gets 1,t_-\gets 2^d$

\noindent\Input{Stream of $x_i\in [2^d]$ chosen in rounds $i = 1, 2, \dots$ by an adversary, followed by labels $y_i$}
\For{$i=1,2,\dots$}
{
Set $t=\floor{\frac{t_-+t_+}{2}}$

Predict $\hat{y_i}=h_{t}(x_i)$ on input $x_i$\;

\If{$\hat{y_i}=1\neq y_i$}
{
Update $t_-\gets x_i$
}

\If{$\hat{y_i}=0\neq y_i$}
{
Update $t_+\gets x_i$
}
}
\end{algorithm}


Whenever the online learner makes a mistake, the set of remaining candidate thresholds is reduced in size by a factor of 2. Since the size of the hypothesis space at the start of the game is $2^d$, the algorithm terminates after at most $d$ mistakes. We will use a variant of this algorithm to efficiently online learn the concept class we construct later. 

\section{Concept Class and its Learnability}

\subsection{Computational Separation between PAC and Private PAC learning}
\cite{BunZ16} proved a computational separation between PAC and Private PAC learning by defining a concept class called $\algstyle{EncThr}$. $\algstyle{EncThr}$ intuitively captures the captures the class of threshold functions where examples are encrypted under an Order Revealing Encryption ($\ORE$) scheme. An $\ORE$ scheme is defined by four algorithms $(\Gen,\Enc,\Dec,\Comp)$. We now describe the functionalities of the algorithms.

 \begin{itemize}
     \item $\Gen(1^\lambda,1^{d})$ is a randomized procedure that takes as inputs a security parameter $\lambda$ and plaintext length $d$, and outputs a secret encryption/decryption key $\sk$ and public parameters $\params$.
	\item $\Enc(\sk,m)$ is a potentially randomized procedure that takes as input a secret key $\sk$ and a message $m\in\{0,1\}^d$, and outputs a ciphertext $c$.
	\item $\Dec(\sk,c)$ is a deterministic procedure that takes as input a secret key $\sk$ and a ciphertext $c$, and outputs a plaintext message $m\in\{0,1\}^d$ or a failure symbol $\bot$.
	\item $\Comp(\params,c_0,c_1)$ is a deterministic procedure that ``compares'' two ciphertexts, outputting either ``$>$'', ``$<$'', ``$=$'', or $\bot$.
 \end{itemize}

Each concept in the class $\algstyle{EncThr}$ is parameterized by a string $r$ that represents the coin tosses of the algorithm $\Gen$ and by a threshold $t\in [N]$ for $N=2^d$, where $d$ represents the length of the plaintext. Let $(\sk^r,\params^r)$ be the secret key and the public parameters output by $\Gen(1^\lambda,1^d)$ when run on the sequence of coin tosses $r$. 
 Formally, a concept in $\algstyle{EncThr}$ parameterized by $t,r$ is defined as follows:

\[f_{t, r}(c,\params) = \begin{cases}
1 & \text{ if } (\params=\params^r)\land (\Dec(\sk^r,c)\neq\perp)\land(\Dec(\sk^r,c)<t) \\
0 & \text{ otherwise.}
\end{cases}\]

Intuitively, each concept $f_{t, r}$ evaluates the threshold function with parameter $t$ on the decryption of the input ciphertext $c$. The particular syntax of the definition handles various technical complications; further discussion appears in~\citep{BunZ16}.

Order revealing encryption enables determining the plaintext ordering given the ciphertexts. While it can be shown that no private PAC algorithm can efficiently learn $\algstyle{EncThr}$, it also not possible for any online learner to learn $\algstyle{EncThr}$ efficiently. To see this, fix a target concept at threshold $t=2^{d-1}$ and random coin tosses $r$. As before, the adversary sends examples to the learner where each unlabled example is of the form $(\ciph,\params^r)$. The learner can only compare the plaintext ordering given the examples chosen by the adversary. The adversary maintains the largest example with label $1$ and smallest example with label $0$ that has been presented to the learner so far. In every round, the adversary picks uniformly at random between the smallest available example in the interval on the left side of the threshold (i.e., the smallest available example between the threshold and the largest positive example) and the largest available example in the interval on the right side of the threshold (i.e., the  largest available example between the threshold and the largest positive example). 

Security of the ORE ensures that for any polynomial time horizon, an efficient learner can do no better than random guessing. Thus, an efficient learner must make super-polynomially many mistakes, so it is not possible to design an efficient online learner for the class $\algstyle{EncThr}$.

To overcome this issue, we design an a concept class similar to $\algstyle{EncThr}$ but where the encryption reveals some additional information about the plaintexts that facilitates online learning. That is, we study function-revealing encryption ($\FRE$) schemes that enable a richer class of functionalities over the underlying plaintexts than just comparisons. Tuning this functionality is crucial for proving our separation. On one hand, we need to reveal more than ordering in order to learn online. On the other end, revealing too much information (e.g., the exact distance between the underlying plaintexts) enables constructing an efficient private PAC learner. In fact, even revealing a multiplicative approximation of the distance also allows for constructing efficient private PAC learners -- later, we sketch how such a learner can be built using the exponential mechanism.

\subsection{Function Revealing Encryption}

Function revealing encryption is a cryptographic scheme that lets users evaluate functions on plaintexts given access to only the corresponding ciphertexts. We use $\FRE$ that lets us evaluate a ``leakage'' function $\lkg$ on plaintexts. We will define a concept class $\algstyle{LEncThr}$ in terms of an abstract $\FRE$ such that appropriate conditions on the leakage function $\lkg$ imply the properties we need for our computational separation.

First, let us describe the general syntax, functionality, and security guarantees we need from $\FRE$, some of which are necessarily non-standard. For instance, the usual definition of $\FRE$ allows for function evaluation with overwhelmingly high probability. But for our applications, we require a stronger notion of correct evaluation. We require the evaluation of the function to succeed with probability $1$ (perfect correctness). Additionally, we require function evaluation over ciphertexts to always behave consistently with decryption. In particular, the evaluation algorithm should output $\bot$ for ciphertexts that are malformed and do not correspond to any messages (strong correctness). 

In general, the leakage function of a function revealing encryption scheme may have arbitrary arity, but we for simplicity we specialize the arity to 3, which captures the way we use such schemes. 
\begin{definition}
A function revealing encryption scheme $\FRE$ with functionality $\lkg$  is a tuple of algorithms $(\Gen,\Enc,\Dec,\Eval)$ where
 \begin{itemize}
     \item $\Gen(1^\lambda,1^d)$ is a randomized procedure that takes as inputs a security parameter $\lambda$ and plaintext length $d$, and outputs a secret encryption/decryption key $\sk$ and public parameters $\params$.
	\item $\Enc(\sk,m)$ is a  randomized procedure that takes as input a secret key $\sk$ and a message $m\in\{0,1\}^d$, and outputs a ciphertext $c$.
	\item $\Dec(\sk,c)$ is a deterministic procedure that takes as input a secret key $\sk$ and a ciphertext $c$, and outputs a plaintext message $m\in\{0,1\}^d$ or a failure symbol $\bot$.
	\item $\Eval(\params,c_0,c_1,c_2)$ is a deterministic procedure that aims to reveal the value of $\lkg$ on the plaintexts associated with $c_0, c_1, c_2$.
 \end{itemize}
\end{definition}

\begin{paragraph}{Correctness.}
A $\FRE$ scheme must satisfy two separate correctness requirements.
\begin{itemize}
    \item \textbf{Correct Decryption:} This is the standard notion of correctness for an encryption scheme, which says that decryption succeeds. 
    For all security parameters $\lambda$ and message lengths $d$, and for all messages $m$, 
    $$\Pr[\Dec(\sk,\Enc(\sk,m))=m:(\sk,\params)\gets\Gen(1^\lambda,1^d)]=1.$$

        \item \textbf{Correct Evaluation:} We require that the evaluation function succeeds.  For every $c_0, c_1, c_2$ in the ciphertext space, define the auxiliary function $\Eval_{\algstyle{\lkg}}^{\algstyle{ciph}}(\sk,c_0,c_1,c_2)$ as follows. It first computes $m_b=\Dec(\sk,\ciph_b)$ for $b\in\{0,1,2\}$. If any of $m_0,m_1$ or $m_2$ is $\bot$, then $\Eval_{\algstyle{\lkg}}^{\algstyle{ciph}}$ computes to $\bot$. If $m_0,m_1,m_2\neq \bot$, then the output is $\lkg(\Dec(\sk,\ciph_0),\Dec(\sk,\ciph_1),\Dec(\sk,\ciph_2))$.

Our definition of ``perfect and strong'' correctness requires that the evaluation function $\Eval$ is always consistent with $\Eval_{\lkg}^{\algstyle{ciph}}$. That is, for all security parameters $\lambda$, all message lengths $d$, and all $c_0, c_1, c_2$ in the ciphertext space,
			\[\Pr\left[\Eval(\params,c_0,c_1,c_2)=\Eval_{\lkg}^{\algstyle{ciph}}(\sk,c_0,c_1,c_2):(\sk,\params)\gets\Gen(1^\lambda,1^d)\right]= 1.\]
\end{itemize}
\end{paragraph}

\begin{definition}[Leakage indistinguishablity security] An $\FRE$ scheme $(\Gen,\Enc,\Dec,\Comp)$ is \emph{statically secure} if, for all polynomial-time adversaries $\cA$,  $|\Pr[W_0]-\Pr[W_1]|$ is negligible, where $W_b$ is the event that $\cA$ outputs $1$ in the following interaction between $\cA$ and a ``challenger'' algorithm:
	\begin{itemize}
		\item $\cA$ produces two message sequences $\{m_1^{(L)},m_2^{(L)},\dots,m_q^{(L)}\}$ and $\{m_1^{(R)},m_2^{(R)},\dots ,m_q^{(R)}\}$ such that for all $i, j, k\in [q]$, $\lkg(m_i^{(L)},m_j^{(L)},m_k^{(L)})=\lkg(m_i^{(R)},m_j^{(R)},m_k^{(R)})$.
		\item The challenger samples $(\sk,\params)\gets\Gen(1^\lambda,1^d)$.  It then reveals $\params$ to $\cA$, as well as $c_1,\dots,c_q$ where 
		\[c_i=\begin{cases}
			\Enc(\sk,m_i^{(L)})&\text{if }b=0\\
			\Enc(\sk,m_i^{(R)})&\text{if }b=1.
		\end{cases}\]
		\item $\cA$ outputs a guess $b'$ for $b$.
	\end{itemize}
\end{definition}

 Here, ``statically" secure refers to the fact that the adversary must submit all of its challenge messages in a single batch, in contrast to an ``adaptive" adversary that may issue challenge messages adaptively depending on the previous ciphertexts received.

\begin{paragraph}{On perfectly and strongly correct evaluation.}
While non-standard, our notions of perfect and strong correctness are important in facilitating our efficient online learner. Note that essentially the same conditions were used in the separation of \cite{BunZ16}. For us, these conditions prevent the adversary in the online learning model from either choosing a value of the randomness in $\Gen$ that causes the $\Eval$ procedure to fail, or by sending the learner malformed ciphertexts as examples.

\end{paragraph}

 Let $N=2^d$ and $[N] = \{1,\dots ,N\}$ be a plaintext space. Let $\lkg:[N]^3\to R$ be a (for now, abstract) leakage function with codomain $R$. 
Let $(\Gen, \Enc, \Dec, \Eval)$ be a statically secure $\FRE$ scheme with functionality $\lkg$, satisfying our perfect and strong correctness guarantees. We define a concept class $\EncThr$, which intuitively captures the class of threshold functions where examples are encrypted under the $\FRE$ scheme. Following~\cite{BunZ16}, let $t\in [N]$. Let $(\sk^r,\params^r)$ be the secret key and the public parameters output by $\Gen(1^\lambda,1^d)$ when run on the sequence of coin tosses $r$. We define





\[f_{t, r}(\ciph,\params) = \begin{cases}
1 & \text{ if } (\params=\params^r)\land (\Dec(\sk^r,c)\neq\perp)\land(\Dec(\sk^r,c)<t) \\
0 & \text{ otherwise.}
\end{cases}\]Note that given $t$ and $r$, the concept $f_{t,r}$ can be efficiently evaluated.


\subsection{Properties of leakage function for separation}

We now describe the properties a leakage function that allow us to separate online learning from private PAC learning. We begin by defining the arity-3 leakage function induced by an abstract ``distance'' function.

Let $\inst=[2^d]$. Let $\dst:\inst\times \inst\to \{-d,\dots ,0,\dots ,d\}$. Think of $\dst$ as an abstract measure of signed distance between inputs, whose absolute value ranges from $\{0, \dots, d\}$. For example, our construction will eventually take $\dst(x, y) = \operatorname{sgn}(x-y) \lfloor \log_2 |x - y| \rfloor$.
The sign of the distance function corresponds to the order of the inputs. That is,
\[\dst(x, y) \begin{cases}
    < 0 & \text{ if } x < y \\
    = 0 & \text{ if } x = y \\
    > 0 & \text{ if } x > y.
\end{cases}\]

\begin{definition}[Distance-induced leakage]\label{def:leak}
Let $\dst$ be a signed distance function as described above. We define the leakage function $\lkg: \inst^3 \to \{<,>,=\}^3\times \{0,1\}$ induced by $\dst$ as follows.
\[
\lkg(x_0,x_1,x_2)=\left(\Comp(x_0,x_1),\Comp(x_1,x_2),\Comp(x_0,x_2),\Ind(\vert\dst(y_0,y_1)\vert<\vert\dst(y_1,y_2)\vert\right),
\]
where $\Comp$ indicates comparison, i.e., $``<",``>"$ or $``="$, and $y_0 \le y_1 \le y_2$ are the inputs $x_0, x_1, x_2$ in sorted order.
\end{definition}

That is, $\lkg$ reveals the pairwise comparisons between the inputs $x_0,x_1,x_2$. It also reveals a bit indicating whether the smaller two plaintexts are closer to each other than the larger two plaintexts.

We now identify the conditions on $\dst$ and $\lkg$ that allow us to prove a computational separation.

\subsubsection{Sufficient Leakage for Online Learning}

The online learner we eventually construct for $\EncThr$ is based on the halving algorithm, which exploits each mistake to noticeably decrease the space of remaining consistent hypotheses. Our sufficient condition for online learnability, stated as follows, ensures that each mistake is guaranteed to lead to progress.

\begin{definition}\label{def:bisection}
    Let $\inst = [2^d]$. Let $\dst$ and $\lkg$ be the functions as defined in Definition \ref{def:leak}. We say $\dst$ has the \emph{bisection property} if for all $x, y, z \in X$ such that $x<y<z$ either $|\dst(y,x)|<|\dst(z,x)|$ or $|\dst(z,y)|<|\dst(z,x)|$. Additionally, $\dst(x,y)=0$ implies that $x=y$.

    We also say that $\lkg$ has the bisection property if it is induced by a distance function $\dst$ with the bisection property.
\end{definition}

\subsubsection{Sufficient condition for Hardness of Private Learning}

\begin{definition}\label{def:swaplog}
    Let $\lkg$ be a leakage function as defined in Definition \ref{def:leak}. We say that $\lkg$ has the \emph{log-invariance} property if there exists a polylogarithmic function $\kappa$ and polynomial $\zeta$ such that for every $n \in \mathbb{N}$, the following holds. With probability at least $1/\zeta(n)$ over a set $S = \{x_1, \dots, x_n\}$ of points drawn uniformly at random from $X$, for every $i \in [n]$, there exists an efficient procedure that outputs a set $R_i$ with $|R_i| \le \kappa(n)$ such that:

    \begin{enumerate}
        \item For all $m_1, m_2 \in S \setminus R_i$ and all $z, z'$ in the interval $(x_{i-1}, x_{i+1})$, we have $\lkg(m_1, m_2, z) = \lkg(m_1, m_2, z')$ and similarly for all permutations of the inputs.
        \item For all $m \in S \setminus R_i$ and all $z_1, z_2, z_1', z_2'$ in the interval $(x_{i-1}, x_{i+1})$, we have $\lkg(m, z_1, z_2) = \lkg(m, z_1', z_2')$ and similarly for all permutations of the inputs.
    \end{enumerate}
\end{definition}

That is, with high probability over uniformly random sets $S$ of plaintexts, the leakage function is robust in the following sense. For every $i$, there is a small (polylogarithmically sized) set of points $R_i$ that can be removed from $S$ such that the leakage function reveals nothing about points in $(x_{i-1}, x_{i+1})$ via their relationship to points in $S \setminus R_i$.

This is an admittedly technical condition. Intuitively, the ``reveals nothing about points in $(x_{i-1}, x_{i+1})$'' condition helps in our lower bound argument to construct pairs of adversarial sequences that respect the constraints imposed by the leakage function. The fact that the set of points $R_i$ has only polylogarithmic size is important for us to use group differential privacy (over removing all of the points in $R_i$) to preserve an inverse polynomial distinguishing advantage.


\section{Efficient Online Learner} \label{sec:online-learner}
We now argue that $\EncThr$ is efficiently online learnable whenever the $\lkg$ function has the bisection property. 
Our online learner $L$ (Algorithm \ref{alg:ol}) operates in two phases. In the first phase, $L$ guesses the label $0$ for all examples until it makes its first mistake. This first mistake reveals the correct set of $\params^r$ that characterizes the fixed concept. Once $L$ recovers the correct set of parameters, it enters a second phase where it runs a variant of the halving algorithm. That is, it keeps track of the largest example with a positive ($1$) label and the smallest example with a negative ($0$) label. 
In every iteration, $L$ matches the parameters of the received example with $\params^r$ to check if the example received is malformed; if so, it predicts label $0$.
Otherwise, if the example has the correct public parameters, $L$ uses the $\Eval$ function to check if the plaintext corresponding to the plaintext is smaller than the plaintext corresponding to largest positive example seen so far, predicting $1$ if this is the case. By the guarantee of perfect and strong correctness of the $\Eval$ algorithm, the learner is guaranteed to predict correctly in this case. Similarly, $L$ labels an example with $0$ if the plaintext corresponding to the example is greater than the the plaintext corresponding to the smallest example with a negative label. Finally, in the case that the example lies between the largest positive example and the smallest positive example, $L$ checks if it is closer to the largest positive example or the smallest negative example using $\Eval$. It then predicts a label according to whichever point it is closer to.

Since $\lkg$ satisfies the bisection property, we know that if $L$ makes a mistake in this final case, then the underlying distance $\dist$ between the largest positive and smallest negative example reduces by $1$. Since $\dist$ takes absolute values between $0$ and $d$, the learner can make at most $d+1$ mistakes in this phase.

We now formalize our online learner $L$. We assume that the first phase is over, so that we've received the correct public $\params^r$. We also assume that we've received at least one positive example and at least one negative example, which will be the case after at most $2$ more mistakes.

Simplifying and abusing notation somewhat, let $\Comp(c_0, c_1)$ below denote the information revealed by the leakage evaluation function $\Eval(\params^r, c_0, c_1, c_2)$ about how (the plaintexts underlying) ciphertexts $c_0, c_1$ compare. Similarly, let $\algstyle{DistComp}(c_0, c_1, c_2)$ denote the information revealed about whether (the plaintexts underlying) $c_0, c_1$ are closer together, or if $c_1, c_2$ are closer.

\begin{algorithm}[!ht]
   \caption{Online learner for $\EncThr$ where $\lkg$ has the bisection property}
\label{alg:ol} 
\textbf{Initialize:} Public parameters $\params^r$, largest positive example $x_+ = (c_+, \params^r)$, and smallest negative example $x_- = (c_-, \params^r)$

\KwIn{Stream of $x_i = (c_i, \params_i)$ in rounds $i = 1, 2, \dots$ chosen by an adversary, followed by labels $y_i$}


\uIf{$\params_i \ne \params^r$ or $\Eval(\params^r,c_+,c_i,c_-)=\bot$}
{
Predict $\hat{y_i}=0$
}

\uElseIf{$\Comp(c_i,c_+) = ``<"$ or $\Comp(c_i,c_+) = ``="$}
{
Predict $\hat{y_i}=1$ 
}
\uElseIf{$\Comp(c_i,c_-) = ``>"$ or $\Comp(c_i,c_-) =``="$}
{
Predict $\hat{y_i}=0$ 
}
\Else
{
\uIf{$\algstyle{DistComp}(c_+,c_i,c_-)=1$}
{Predict $\hat{y}_i = 1$

\If{$\hat{y_i}\neq y_i$}
{
Update $\ciph_- \gets \ciph_i$
}
}
\uIf{$\algstyle{DistComp}(c_+,c_i,c_-)=0$}
{Predict $\hat{y}_i = 0$

\If{$\hat{y_i}\neq y_i$}
{
Update $\ciph_+ \gets \ciph_i$
}
}
}
\end{algorithm}

\subsection{Analysis}

We will use a potential argument to show that the online learner makes at most $d+4$ mistakes. The potential function is simply the absolute value of the distance between the plaintexts underlying $c_+$ and $c_-$. That is, for a fixed target function (and hence, choice of secret key $\sk$), define
\[D(c_+, c_-) = |\dst(\Dec(\sk,c_+), \Dec(\sk,c_-)|.\]
(If either decryption fails, then set $D(c_+, c_-) = \bot$.)

We know that $t\in [2^d]$ by the definition of the concept class. 
We assume that the learner knows the correct set of $\params$ in our analysis since it takes at most one mistake to discover this.

Assume without loss of generality that $c_-=\Enc(\sk,2^d)$ and $c_+=\Enc(\sk,1)$, as  different choices of the underlying plaintexts will only improve the analysis below. Note that the learner makes at most two mistakes to get these initial values.
Thus, we assume that $D(c_+,c_-)\leq d$ at the beginning of the algorithm.
Because of the bisection property of the $\dst$ function, we know that every time the learner makes a mistake (and hence, either $c_+$ or $c_-$ gets updated), the value of $D(c_+,c_-)$ shrinks by at least one. Also, $D(c_+,c_-)$ can never fall below $0$.

\begin{lemma}
If Algorithm~\ref{alg:ol} has made $m$ mistakes, then $D(c_+,c_-) \leq d-m$.
\end{lemma}
\begin{proof}
We prove this statement by induction on $m$. As our base case, take $m = 0$; then $D(c_+,c_-) \le d$ as observed above.

Now suppose the claim holds for $m$ mistakes. We now show that after an additional mistake, $D(c_+',c_-') \leq d-(m+1)$ where $c_+'$ and $c_-'$ are the updated examples.

Let $i$ be the iteration in which the mistake is made. Following the algorithm, there are two cases we need to analyze.

\paragraph{Case 1: $D(c_+, c_i) < D(c_i, c_-)$.} Here, the learner incorrectly predicted $\hat{y}_i = 1$, giving the update $c_-' = c_i$. The bisection property of $\dst$ guarantees that $D(c_+, c_i) = \min\{D(c_+, c_i),  D(c_i, c_-)\} < D(c_-, c_+)$. Since $D(c_-, c_+) \le d-m$ by our inductive hypothesis, we have that $D(c_+', c_-') = D(c_+, c_i) \le d - (m+1)$.

\paragraph{Case 2: $D(c_+, c_i) \ge D(c_i, c_-)$.} In this case, the learner incorrectly predicted $\hat{y}_i = 0$, resulting in the update $c_+' = c_i$. Then again, the bisection property guarantees that $D(c_i, c_-) \le \min\{D(c_+, c_i),  D(c_i, c_-)\} < D(c_-, c_+)$. Since $D(c_-, c_+) \le d-m$ by our inductive hypothesis, we have that $D(c_+', c_-') = D(c_i, c_-) \le d - (m+1)$.





    

This completes the proof of the inductive step.
\end{proof}

Combining this lemma with the fact that $D$ can never fall below $0$, we obtain the following.

\begin{theorem}
Suppose $\lkg$ has the bisection property. Then Algorithm \ref{alg:ol} (with preprocessing described above) learns the associated concept class $\EncThr$ with mistake bound $d+4$ and polynomial runtime per example.
\end{theorem}



\section{Hardness of Privately PAC Learning  $\EncThr$} \label{sec:hardness}
We now prove that there is no computationally efficient private PAC learner for $\EncThr$ whenever the leakage function $\lkg$ is log-invariant (Definition~\ref{def:swaplog}). 
The goal of this section is to prove the following statement.

\begin{theorem}
Let $\FRE = (\Gen,\Enc,\Dec,\Eval)$ be a statically secure $\FRE$
scheme where $\Eval$ efficiently evaluates a log-invariant leakage function $\lkg$. Then there is no polynomial-time differentially private PAC learner for the associated concept class $\EncThr$. 
\end{theorem}
We first provide a proof sketch of the theorem statement. The idea is to show that if there were an accurate, efficient, differentially private PAC learner for $\EncThr$, then we could use it to construct an efficient adversary that violates the $\FRE$ scheme. 

\paragraph{Implications of accuracy of the learner.}
Let $N=[2^d]$ be the space of the plaintexts. Fix the target threshold to $t=N/2$ and construct labeled examples $S=\{(x_1=(\params,\Enc(\sk,m_1)),y_1)$ \\ $,\dots ,(x_n=(\params,\Enc(\sk,m_n)),y_n)\}$ by sampling plaintexts $\{m_1,\dots m_n\}$ uniformly at random. We can break up the plaintexts space into buckets of the form $B_i=[m_i,m_{i+1})$. Suppose $L$ is an $(\alpha,\beta)$-accurate PAC learner. Then with probability at least $1-\beta$, the hypothesis $h$ produced by $L$ can distinguish encryptions of messages $m< t$ from encryptions of messages $m\ge t$ with accuracy at least $(1-\alpha)$. By an averaging argument, there exists an index $i\in [n]$ such that the hypothesis can distinguish between consecutive points sampled from bucket $B_{i-1}$ versus points sampled from $B_{i}$ with probability at least $(1-\alpha)/n$. The ability of the learner to distinguish between consecutive points is crucial for designing the adversary for the $\FRE$ security game. Knowing this, let's see how we can construct an adversary that violates the security of the $\FRE$. A natural first attempt is to construct a pair of challenge sequences $m_1<\dots m_{i-1}<m_i^{(L)}< m_{i+1}\dots m_n$ and $m_1<\dots m_{i-1}<m_i^{(R)}< m_{i+1}\dots m_n$, where $m_i^{(L)}$ is randomly chosen from $B_{i-1}$ and $m_i^{(R)}$ is randomly chosen from $B_i$. Let's assume for now that for any indices of points $u,w\in [n]$, $\lkg(m_u,m_i^{(L)},m_w)= \lkg(m_u,m_i^{(R)},m_w)$ (we will show later how privacy of the learner helps us achieve this). Then if $h$ can distinguish $B_{i-1}$ from $B_i$ , the adversary can distinguish the two sequences. Unfortunately, this approach doesn't quite work. The hypothesis $h$ is only guaranteed to distinguish $B_{i-1}$ from $B_i$ with probability $(1-\alpha)/n$. If $h$ fails to distinguish the buckets or distinguishes them in the opposite direction then, the adversary’s advantage is lost.

Thus, following the approach of~\cite{BunZ16}, we consider sequences of challenge messages that differ on two messages. For the ``left" challenge sequence our adversary samples two messages from the same of either $B_{i-1}$ or $B_i$. For the ``right" challenge sequence our adversary samples from one message from each bucket $B_{i-1}$ and $B_i$. Both challenge sequences are completed with the same messages $m_1,\dots m_{i-1},m_{i+1},\dots m_n$. Let $c^0$ and $c^1$ be the ciphertexts corresponding to the messages that are different between the two sequences. If the learned hypothesis $h$ agrees on $c^0$ and $c^1$, then the challenge messages are more likely to be from the same bucket. If $h$ disagrees, then the challenge sequences are more likely to be from different buckets. With this setup, any advantage $h$ enjoys over random guessing when the learner succeeds is preserved even if it has no advantage when the learner fails. 

The difficulty now is ensuring that the ``left'' and ``right'' challenge messages are indistinguishable with respect to the leakage function $\lkg$. Sampling multiple messages from each bucket makes this task harder still. We now explain how differential privacy helps us overcome this issue.

\paragraph{Use of differential privacy.} Differential privacy of the learner permits us to swap out points in its input dataset while preserving its distinguishing advantage. Obviously, deleting the example corresponding to $m_i$ is essential to having any hope of constructing a pair of challenge sequences that agree on $\lkg$. But $\lkg$ imposes more stringent constraints. The log-invariance property of $\lkg$ that we define ensures that by removing only a polylogarithmic number of samples (which, by group privacy, doesn't hurt our distinguishing advantage too much), we can indeed construct such challenge sequences.


We now formalize the ideas described in the proof sketch. First, we show that an accurate learner is likely to output a hypothesis that can distinguish between two adjacent buckets. 



\begin{lemma}\label{lemma:jump}
Consider the concept class  $\EncThr$. Let $L$ be a $(\alpha=1/4,\beta=1/4)$-accurate PAC learner for $\EncThr$. Fix any pair $(\sk, \params)$ in the range of $\Gen$ and an (encrypted) threshold concept with $t=N/2$.

Let $S=\left\{ \left(x_{1}=(\params,\Enc(\sk,m_1)),y_{1}\right),\dots,\left(x_{n}=(\params,\Enc(\sk,m_n)),y_{n}\right)\right\}$ where $m_i$ are sampled uniformly at random and $y_i=f_t(m_i)$. Then there exists an $i\in\left[n\right]$ such that 
\[
\Pr\left[\left\vert\Pr_{m\sim B_{i}}\left[h\left(\Enc\left(m,\sk\right)\right)=1\right]-\Pr_{m\sim B_{i+1}}\left[h\left(\Enc\left(m,\sk\right)\right)=1\right]\right\vert\geq\frac{1}{2n}\right]\geq\frac{3}{4n}.
\]
Here, the outer probability is over the randomness of the samples $S$ and the randomness of the learner, $h \gets L(S)$, and each $B_i=[m_i,m_{i+1})$.
\end{lemma}

\begin{proof}
Fix a dataset $S$ and let $h$ be the hypothesis produced by the learner on input $S$. Let $B_i=[m_i,m_{i+1})$, $\ell_i=\vert B_i\vert /2^d$ and $p_i=\Pr_{m\sim B_{i}}\left[h\left(\Enc\left(m,\sk\right)\right)=1\right]$ for each $i\in [n]$. Let $k$ be the index of the bucket where the threshold $t$ lies.

Accuracy of the learner implies that with
probability at least $1-\beta \ge 3/4$ over the sample $S$ and the learner's coin tosses, we have
\[\sum_{i=1}^{k-1}p_{i}\ell_i+\sum_{i=k+1}^{n}\left(1-p_{i}\right)\ell_i+p_{k}\ell_a+\left(1-p_{k}\right)\ell_b\geq1-\alpha = \frac{3}{4},\]
where $\ell_a = |t-m_k|/2^d$ and $\ell_b = |m_{k+1}-t-1|/2^d$.

We claim that if this is the case, then there exist indices $i < j$ such that $\left|p_{i}-p_{j}\right|\geq 1/2$. To see this, assume instead for the sake of contradiction that there exists $p$ such that for all $i\in\left[n\right]$, $\vert p_{i}-p\vert<1/4$. Then
\begin{align*}
\sum_{i=1}^{k-1}p_{i}\ell_i+\sum_{i=k+1}^{n}\left(1-p_{i}\right)\ell_i+&p_{k}\ell_a+\left(1-p_{k}\right)\ell_b \\
& <\left(p+1/4\right)\underbrace{\sum_{i = 1}^{k-1}\ell_i+\ell_a}_{1/2}+\left(1-\left(p-1/4\right)\right)\underbrace{\sum_{i= m+1}^n\ell_i+\ell_b}_{1/2}\\
 & =\frac{1}{2}\left(p+1/4+1-\left(p-1/4\right)\right)\\
 & =\frac{1}{2}+\frac{1}{4} = \frac{3}{4}.
\end{align*}
This contradicts our assumed accuracy of the learner.

Thus, we have shown that
\[
\Pr\left[ \exists i < j \text{ s.t. } \ \left\vert\Pr_{m\sim B_{i}}\left[h\left(\Enc\left(m,\sk\right)\right)=1\right]-\Pr_{m\sim B_{j}}\left[h\left(\Enc\left(m,\sk\right)\right)=1\right]\right\vert\geq\frac{1}{2}\right]\geq\frac{3}{4}.
\]
More compactly, if we denote $\Pr_{m\sim B_{i}}\left[h\left(\Enc\left(m,\sk\right)\right)=1\right]$ by $p_i$ for all $i\in [n]$, that is,

\[
\Pr\left[\exists i<j \text{ s.t. }  \  \left\vert p_i-p_j\right\vert\geq \frac{1}{2}\right]\geq \frac{3}{4}.
\]
If for some $i < j$ we have $\left\vert p_i-p_j\right\vert\geq 1/2$, then the triangle inequality implies $\left\vert p_i-p_{i+1}\right\vert + \left\vert p_{i+1}-p_{i+2}\right\vert +\dots +\left\vert p_{j-1}-p_{j}\right\vert\geq\frac{1}{2}$. By averaging, this in turn implies that there exists an index $k$ where $i \le k \le j-1$ for which $|p_k - p_{k+1}| \ge 1/2n$. Thus we have,
\[
\Pr\left[ \exists k \text{ s.t. }  \ \left\vert p_k-p_{k+1}\right\vert \geq\frac{1}{2n}\right]\geq\frac{3}{4}.
\]

Using the union bound we get,
\[
\Pr\left[\left\vert p_1-p_2\right\vert\geq \frac{1}{2n}\right]+\Pr\left[\left\vert p_2-p_3\right\vert\geq \frac{1}{2n}\right]+\dots +\Pr\left[\left\vert p_{n-1}-p_n\right\vert\geq \frac{1}{2n}\right]\geq\frac{3}{4}.
\]
Now by averaging, we conclude that there exists an $i\in[n]$ such that 
\[
    \Pr\left[\left\vert p_i-p_{i+1}\right\vert\geq \frac{1}{2n}\right]\geq \frac{3}{4n}.
\]
Unpacking the definition of $p_i$, equivalently, there exists an $i\in[n]$ such that
\[
\Pr\left[\left\vert\Pr_{m\sim B_{i}}\left[h\left(\Enc\left(m,\sk\right)\right)=1\right]-\Pr_{m\sim B_{i+1}}\left[h\left(\Enc\left(m,\sk\right)\right)=1\right]\right\vert\geq\frac{1}{2n}\right]\geq\frac{3}{4n}.
\]
\end{proof}

We now use group privacy to show that if we switch $\kappa(n) $ points from $S$ to obtain a new dataset $S_i$ to be used as input to the learner, then the gap above still (approximately) holds. 


 \begin{lemma}\label{lemma:switch}
  Let $\epsilon \le 1/\kappa(n)$ and $\delta \le 1/10n$. Let $L$ be a $(\alpha=1/4,\beta=1/4)$-accurate and $(\epsilon, \delta)$-differentially private PAC learner for the concept class $\EncThr$. Consider \\$S=\left\{ \left(x_{1}=(\params,\Enc(\sk,m_1)),y_{1}\right),\dots
  ,\left(x_{n}=(\params,\Enc(\sk,m_n)),y_{n}\right)\right\}$ where $m_i$ are sampled uniformly at random. 
  
  Let $i$ be the index guaranteed by Lemma~\ref{lemma:jump} for which
\[
\Pr\left[\left\vert\Pr_{m\sim B_{i}}\left[h_{S}\left(\Enc\left(m,\sk\right)\right)=1\right]-\Pr_{m\sim B_{i+1}}\left[h_{S}\left(\Enc\left(m,\sk\right)\right)=1\right]\right\vert\geq\frac{1}{2n}\right]\geq\frac{3}{4n},
\]
where the outer probability is taken over the sample and the coins of the learner, and $h_S \gets L(S)$.

Let $S_i=S\backslash R_i$ where $R_i$ is any set such that $\vert R_i\vert \le \kappa(n)$. Then
\[
\Pr\left[\left\vert\Pr_{m\sim B_{i}}\left[h_{S_i}\left(\Enc\left(m,\sk\right)\right)=1\right]-\Pr_{m\sim B_{i+1}}\left[h_{S_i}\left(\Enc\left(m,\sk\right)\right)=1\right]\right\vert\geq\frac{1}{2n}\right]\geq\frac{1}{10n},
\]
where $h_{S_i} \gets L(S_i)$.
\end{lemma}

\begin{proof}
Consider a postprocessing $A$ of the learner $L$ defined as follows: 
\[
A(h)=\\ \left\vert\Pr_{m\sim B_{i}}\left[h\left(\Enc\left(m,\sk\right)\right)=1\right]-\Pr_{m\sim B_{i+1}}\left[h\left(\Enc\left(m,\sk\right)\right)=1\right]\right\vert.
\]
Let $T$ be the set of outcomes for which $A(h) \ge 1/2n$. By hypothesis, we have $\Pr\left[A(L(S))\in T\right]\geq 3/4n$, where the probability is taken over the sample and the coins of the learner.

Using the post-processing property of differentially private mechanisms, we get that $A\circ L$ is $(\epsilon,\delta)$-differentially private. Switching the input dataset from $S$ to $S_i$ and using group privacy, we get
\begin{align*}
  \Pr\left[A\left(L(S)\right)\in T\right]\leq e^{\vert R_i\vert\epsilon}\Pr\left[A\left(L(S_i)\right)\in T\right] + \frac{e^{\vert R_i\vert\epsilon}-1}{e^\epsilon-1}\cdot\delta.
\end{align*}
Since $\vert R_i\vert \le \kappa(n)$, then as long as $\epsilon\le 1/\kappa(n)$ and $\delta \le 1/4n$, we get $\Pr\left[A\left(L(S_i)\right)\in T\right]\geq\frac{1}{10n}$.
\end{proof}

\begin{algorithm}
\begin{enumerate}
\item Set $t=N/2$ and choose uniformly at random $i\sim [n]$.
\item Sample $n$ points uniformly at random and permute in increasing order to get $P=\left(m_{1},\dots,m_{i-1},m_{i},m_{i+1},\dots m_{n}\right)$
\item Construct pairs $\left(m_L^{0},m_L^{1}\right)$ and $\left(m_R^{0},m_R^{1}\right)$
as follows. Let $B_{i-1} = [m_{i-1},m_i)$ and $B_i= [m_i,m_{i+1})$. Sample $m_L^{0}<m_L^{1}$ from $B_j$ for a random choice of $j\in \{i-1,i\}$ and sample $m_R^0$ from $B_{i-1}$ and $m_R^1$ from $B_i$.

\item  Let $P_i=P\setminus R_i$, where $|R_i| \le \kappa(n)$ as guaranteed by log-invariance. Challenge on the pair of sequences $P_i\cup \{m_L^{0},m_L^{1}\}$ and $P_i\cup \{m_R^{0},m_R^{1}\}$ (in sorted order) and receive the sequence of ciphertexts $\left(c_{1},\dots,c_i^{0},c_i^{1},\dots c_{n-|R_i|}\right)$.
\item Remove $c_i^{0},c_i^{1}$ from the set of ciphertexts and construct a dataset by attaching public parameters and labels $y_j = f_t(m_j)$, i.e. $S_i=\{(x_1=(c_1,\params),y_1),\dots ,(x_{n-|R_i|}=(c_{n-|R_i|},\params),y_{n-|R_i|})\}$. Obtain $h\gets_R L(S_{i})$.
\item Set $x_i^0 = (c_i^0, \params)$ and $x_i^1 = (c_i^1, \params)$. Guess $b'=0$ if $h(x_i^0)=h(x_i^1)$. Guess $b'=1$ otherwise.
\end{enumerate}
\caption{Adversarial strategy using DP-PAC Learner}
\label{alg:game}
\end{algorithm}

\begin{theorem} \label{thm:DP-breaks-FRE}
    Let $L$ be an $(\alpha=1/4,\beta=1/4)$-accurate and $(\eps, \delta)$-differentially private PAC learner with $\epsilon \le 1/\kappa(n)$ and $\delta \le 1/4n$ for the concept class $\EncThr$, where the underlying $\FRE$ scheme is instantiated using a log-invariant leakage function $\lkg$. Then there exists an adversary that wins the security game of the $\FRE$ with advantage at least $1/\poly(n)$.
\end{theorem}

\begin{proof}
We describe our adversarial strategy as Algorithm~\ref{alg:game}. 

Note that a randomly chosen $i\in[n]$ meets the guarantee described in Lemma \ref{lemma:jump} with probability at least $1/n$. Moreover, since $\lkg$ is log-invariant, we know that there exists an efficient procedure that outputs $P_i$ as described in step 4 with probability at least $1/\zeta(n)$  for some polynomial $\zeta$. (Otherwise, our adversary can just output a random guess.) Putting these together, we get the following guarantee.

\begin{align}\label{eq:switching}
\Pr\left[\left\vert\Pr_{m\sim B_{i}}\left[h_{S_i}\left(\Enc\left(m,\sk\right)\right)=1\right]-\Pr_{m\sim B_{i+1}}\left[h_{S_i}\left(\Enc\left(m,\sk\right)\right)=1\right]\right\vert\geq\frac{1}{2n}\right]&\geq\frac{1}{10 n}\cdot \frac{1}{n}\cdot \frac{1}{\zeta(n)}\nonumber\\
&=\frac{1}{\poly(n)}.
\end{align}



Now fix a realization of $S$.
As before, let $p_i = \Pr_{m\sim B_i}[h_{S_i}(\Enc(m,\sk))=1]$ for each $i\in [n]$. The advantage of the adversary in the security game under this realization of $S$ is
    \begin{align*}
\Pr\left[b'=b\right] & =\frac{1}{2}\left(\Pr\left[h(x_i^0)=h(x_i^1)\mid b=L\right]+\Pr\left[h(x_i^0)\neq h(x_i^1)\mid b=R\right]\right)\\
 & =\frac{1}{2}\left(\frac{1}{2}\left(p_{i}^{2}+(1-p_{i})^{2}+p_{i+1}^{2}+(1-p_{i+1})^{2}\right)+\left(1-p_{i}p_{i+1}-\left(1-p_{i}\right)\left(1-p_{i+1}\right)\right)\right)\\
 & =\frac{1}{2}\left(\frac{1}{2}\left(2p_{i}^{2}+2p_{i+1}^{2}+2-4p_{i}p_{i+1}\right)\right)\\
 & =\frac{1}{2}\left(1+\left(p_{i}-p_{i+1}\right)^{2}\right).
\end{align*}

Thus, if $p_i-p_{i+1}\geq \frac{1}{2n}$, then the advantage is at least $\frac{1}{4n^2}$. For other values of $p_i$ and $p_{i+1}$, the advantage is still non-negative. From Equation \ref{eq:switching}, we know that $p_i-p_{i+1}\geq \frac{1}{2n}$ is with probability at least $1/\poly(n)$. Hence, the overall advantage of the adversary over the random choice of $S$ is at least $1/\poly(n)$.
\end{proof}

\newcommand{\num}{\cI}
\newcommand{\total}{\mathsf{count}}
\newcommand{\inter}{U}
\section{Identifying an Appropriate Leakage function: $\tfld$}\label{sec:strong ORE}

\subsection{Results with $\tfld$ leakage}

In this section, we describe an explicit distance and induced leakage function, denoted $\tfld$, that has both the bisection and log-invariance properties.

\begin{definition}
    Let $\inst = [2^d]$ and $m_0,m_1\in \inst$. We define $\fld$ as a function that reveals the signed floor-log distance between the inputs i.e. 
    \[\fld(m_0,m_1) = \begin{cases}
        0  & \text{ if } m_0 = m_1 \\
        \quad\lfloor\log(m_0-m_1)\rfloor + 1 & \text{ if } m_0 > m_1 \\
        -\;\lfloor\log(m_1-m_0)\rfloor - 1 & \text{ otherwise.}
\end{cases}\]

The induced leakage function $\tfld$ is thus
	\begin{align*}
         \tfld(\plain_0,\plain_1,\plain_2) = \big(\algstyle{Comp}(m_0,m_1),\algstyle{Comp}(m_1,m_2),&\algstyle{Comp}(m_0,m_2), \\
        &\Ind(\floor{\log\vert\plain_1-\plain_0\vert} <\floor{\log\vert\plain_2-\plain_1\vert})  \big)\\
    \end{align*}
    where $\algstyle{Comp}(m_0,m_1)$ reveals if $m_0<m_1$ or $m_0>m_1$ or $m_0=m_1$.
\end{definition}

\begin{lemma}
    The floor-log distance function $\fld$ (and hence, its induced leakage function $\tfld$) has the bisection property.
\end{lemma}

\begin{proof}

    First, the definition of $\fld$ ensures that it reveals the ordering of inputs. Moreover, it is easy to see that $0\leq \fld(m_0,m_1)\leq d$.

    We now argue that for any $m_0<m_1<m_2$, either $\fld(m_1,m_0)<\fld(m_2,m_0)$ or $\fld(m_2,m_1)<\fld(m_2,m_0)$.    Let $z = \fld(m_2,m_0)$ which implies that $2^{z-1}\leq m_2-m_0 < 2^{z}$. For the sake of contradiction, assume that neither $\fld(m_1,m_0)$ nor $\fld(m_2,m_1)$ are less than $z$. It is immediate that neither $\fld(m_1,m_0)$ or $\fld(m_2,m_1)$ can be greater than $z$. In the case that both of them are equal to $z$,  we would have $2^{z-1}\leq m_2-m_1 < 2^{z}$ and $2^{z-1}\leq m_1-m_0 < 2^{z}$. This implies that $m_2-m_0\geq 2^{z}$ which is a contradiction. 
\end{proof}

\begin{corollary} \label{cor:online}
 Algorithm \ref{alg:ol} learns $\EncThr$ under leakage function $\tfld$ with mistake bound $d+4$ and polynomial runtime per example.
\end{corollary}


We now argue that $\tfld$ has the log-invariance property. First, we establish two helpful senses in which uniformly random points are well-spread.

\begin{lemma} \label{lem:small-int}
Call a multiset of points $S = \{m_1, \dots, m_n\} \subseteq [2^d]$ \emph{regular} if for every $i$, we have $|A_i \cap S| \le 50 \log^2 n$ where
\[A_i = \left\{x \in [2^d] \mid 2^z - \frac{4 \log n \cdot 2^d }{n} \le |x - m_i| \le 2^z + \frac{4 \log n \cdot 2^d}{n} \text{ for some } z \in \{0, 1, \dots, d-1\}\right\}.\]
A uniformly random set of points $S$ is regular with probability at least $1 - 1/n$.
\end{lemma}

\begin{proof}
First observe that regardless of the realization of $m_i$, we have
\[|A_i| \le 2 \cdot \frac{4 \log n \cdot 2^d}{n} + \sum_{z = \log(4\log n\cdot 2^d / n)}^d 4 \cdot \frac{4\log n \cdot 2^d}{n} \le \frac{24 \log^2 n \cdot 2^d}{n}.\]
Given $m_i$, the remaining points in $S$ remain uniformly random. Therefore, each of these $n-1$ remaining points intersects $A_i$ independently with probability $24 \log^2 n / n$. By a Chernoff bound, the probability that more than $50 \log^2 n$ of these points intersects $A_i$ is at most $e^{-4\log^2 n}$. Taking a union bound over $i = 1, \dots, n$ completes the proof.
\end{proof}

\begin{lemma}\label{lemma:min}
Let $S = \{m_1, \dots, m_n\}$ consist of $n$ points drawn uniformly at random from $[2^d]$, and arranged in nondecreasing order. With probability at least $1-\frac{1}{n}$ over the sampling of $S$, for all $i\in \{0, 1, \dots, n\}$, we have $\vert B_i\vert \le \frac{4\log n\cdot 2^d}{n}$, where $B_i = [m_i,m_{i+1})$ and $m_0=0$ and $m_{n+1}=2^d$. 
\end{lemma}

\begin{proof}
    Consider dividing $[2^d]$ into disjoint consecutive intervals of length $\frac{2\log n\cdot 2^d}{n}$. We show that with high probability over a random sampling of $n$ points, every interval will contain at least one sampled point.

    Let $\inter_1,\dots ,\inter_{\num}$ denote these disjoint intervals, where $\num=\frac{n}{2\log n}$. Our goal is to show that $\Pr[\forall i\in[\num]:\total(\inter_i)\geq 1]\geq 1-\frac{1}{n}$, where $\total(\inter_i)$ evaluates the total number of sampled points in the interval $\inter_i$. 


    Fix some $i\in[\num]$. Then 
\[      \Pr[\total(\inter_i)<1]=\Pr[\total(\inter_i)=0]= \left(\frac{\num-1}{\num}\right)^n\]

    By the union bound, we get
    \begin{align*}
        \Pr[\exists i\in[\num]:\total(\inter_i)<1]&\leq\num\cdot \left(\frac{\num-1}{\num}\right)^n\\
        &=\num\cdot \left(1-\frac{1}{\num}\right)^n\\
        &\leq \frac{\num}{\exp\left(n/\num\right)}\\
        &\leq \frac{n}{n^2}=\frac{1}{n}.
    \end{align*}
    
    Taking the complement of this event, we conclude that with probability at least $1-\frac{1}{n}$ over the sampling of $S$, for all $i\in\{0,1,\dots ,n\}$, we have $\vert B_i\vert\leq \frac{4\log n\cdot 2^d}{n}$.
\end{proof}

\begin{lemma} \label{lem:fldspread}
Let $S=\{m_1,\dots m_n\}$ consist of points sampled uniformly at random from $[2^d]$. Then with probability at least $1 - 2/n$ over the sampling,
for all $i\in\left[n\right]$, there exists an efficiently computable set of points $R_{i}$ with $\vert R_i\vert \le 50\log^2 n$
such that for all $y\in S\setminus R_{i}$, $\mathsf{fld}\left(y,m_{i-1}\right)=\mathsf{fld}\left(y,m_{i+1}\right)$.
\end{lemma}

\begin{proof}
Let $G = 4 \log n \cdot 2^d / n$. By Lemma~\ref{lem:small-int}, we have that $S$ is regular with overwhelming probability. By Lemma~\ref{lemma:min}, we have that all bucket lengths $|B_i| \le G$ with high probability. We will show that if both of these events hold, then we can construct the appropriate sets $R_i$. For each $i$, define $R_{i} = A_i \cap S$, using the notation from Lemma~\ref{lem:small-int}, which has the requisite size.


Fix some $m\in S\setminus R_{i}$ such that $m \le m_{i-1}$. (We can make a symmetric argument for $m\ge m_{i+1}$). 
Let $p = \fld\left(m,m_{i}\right)$. By properties of floor-log
distance and the construction of the set $R_{i}$ we know that $\left|m-m_{i}\right|<2^{p}-G.$
By the triangle inequality, we have $\vert m-m_{i+1}\vert\leq\vert m-m_{i}\vert+\vert m_{i}-m_{i+1}\vert<2^{p}-G + G=2^{p}$.

On the other hand, from the construction of $R_i$, we have that $\left|m-m_{i+1}\right|\geq \left|m-m_{i}\right| > 2^{p-1}+ G$. This implies that $\fld\left(m,m_{i+1}\right)=p = \fld\left(m,m_{i}\right)$. 

We now argue that $\fld(m,m_{i-1})=p$ as well. It is easy to see that $\vert m-m_{i-1}\vert < \vert m-m_i\vert < 2^{p}-G$. On the other hand,
\begin{align*}
    \vert m-m_{i-1}\vert &= \vert m-m_i\vert -\vert m_{i-1}-m_i\vert \\
                         &\geq \vert m-m_i\vert - G \\
                         & > 2^{p-1} + G - G = 2^{p-1}.
\end{align*}

This proves that $\fld(m,m_{i-1})=\fld(m,m_i)=\fld(m,m_{i+1})=p$ for all $m\in S\setminus R_i$.

\end{proof}

\begin{corollary} \label{cor:tfld}
    The leakage function $\tfld$ induced by $\fld$ has the log-invariance property with $\kappa(n) = 50\log^2 n$.
\end{corollary}

\begin{proof}
    Let $S$ consist of uniformly random samples and construct the sets $R_i$ as in Lemma~\ref{lem:fldspread}. It immediately follows that for all $m_1,m_2\in S\setminus R_i$ and all $z,z'$ in the interval of $(m_{i-1},m_{i+1})$, we have $\tfld(m_1,m_2,z)=\tfld(m_1,m_2,z')$ since $\fld(m_1,z)=\fld(m_1,z')$ and $\fld(m_2,z)=\fld(m_2,z')$. It also follows that for all; $m\in S\setminus R_i$ and all $z_1,z_2,z_1',z_2'$ in the interval $(m_{i-1},m_{i+1})$, we have $\tfld(m,z_1,z_2)=\tfld(m,z_1',z_2')$.
\end{proof}

Combining Corollary~\ref{cor:tfld} with Theorem~\ref{thm:DP-breaks-FRE} yields the following hardness result.

\begin{corollary} \label{cor:hardness}
    Let $\algstyle{FRE}=(\Gen,\Enc,\Dec,\Comp)$ be a statically secure function revealing encryption scheme with leakage function $\tfld$. Then there is no $(\alpha=1/4,\beta=1/4)$-accurate and $(\eps, \delta)$-differentially private PAC learner for the concept class $\EncThr$ with $\epsilon=1/50 \log^2 n$ and $\delta=1/4n$.
\end{corollary}

We now state a result from~\cite{balle2018privacy} that (building on the ``secrecy of the sample'' argument from~\cite{KasiviswanathanLNRS11}) that enables efficiently reducing $\eps$ parameter of a differentially private algorithm using random sampling. We then use this theorem to obtain our main separation result (Theorem~\ref{thm:sos}).

\begin{theorem}
    Fix $\eps\leq 1$ and let $\cA$ be an $(\eps,\delta)$-differentially private algorithm operating on datasets of size $m$. For $n\geq 2m$, construct an algorithm $\tilde{\cA}$ that, on input a dataset $D$ of size $n$, subsamples (without replacement) $m$ records from $D$ and runs $\cA$ on the result. Then $\tilde{\cA}$ is $(\tilde{\eps},\tilde{\delta})$--differentially private for
    \begin{align*}
        \tilde{\eps}=\frac{\left(e^{\eps}-1\right)m}{n}\quad\text{and}\quad\tilde{\delta}=\frac{m}{n}\cdot \delta.
    \end{align*}
\end{theorem}


\begin{theorem}\label{thm:sos}
    Let $\algstyle{FRE}=(\Gen,\Enc,\Dec,\Comp)$ be a statically secure function revealing encryption scheme with leakage function $\tfld$. Define the associated concept class $\EncThr$. Then $\EncThr$ is online learnable in polynomial time with a polynomial mistake bound. However, there is no $(\alpha=1/4,\beta=1/4)$-accurate and $(\eps = 1, \delta = 1/4n)$-differentially private PAC learner for $\EncThr$.
\end{theorem}

\begin{proof}
    Corollary~\ref{cor:online} shows the existence of an efficient online learner for $\EncThr$ under the leakage function $\tfld$ with mistake bound $d+4$. This proves the efficient learnability of $\EncThr$ in the online learning model.

    We now argue about the computational hardness of $\EncThr$.
    Corollary~\ref{cor:hardness} shows that there is no $(\alpha=1/4,\beta=1/4)$-accurate and $(\eps=1/50\log^2n,\delta=1/4n)$-differentially private PAC learner for the concept class $\EncThr$, where the sample size $n$ is any polynomial in the problem description size $d$. We now use the result that $\eps$ can be amplified efficiently by subsampling to show the non-existence of a learner with same accuracy guarantess but worse privacy guarantees. 

    For the sake of contradiction, assume the existence of a learner $\tilde{L}$ for $\EncThr$ that is $(\alpha=1/4,\beta=1/4)$-accurate and $(\eps = 1, \delta = 1/4m)$-differentially private using $m=\poly(d)$ samples. We now construct a learner $L$  for $\EncThr$ that is $(\alpha=1/4,\beta=1/4)$-accurate and $(\eps=1/50\log^2n,\delta=1/4n)$-differentially private using $n = 100m\log^2m = \poly(d)$ samples. The learner $L$ subsamples a dataset of size $m$ without replacement from its input and runs $\tilde{L}$ on it. 

    From the guarantees of Theorem~\ref{thm:sos}, we obtain that $L$ is $(\eps=1/50\log^2n,\delta=1/4n)$-differentially private for sufficiently large $n$. Moreover, since we are running $\tilde{L}$ on subsamples that were sampled without replacement from a dataset whose elements were sampled in an i.i.d. fashion, the output of $\tilde{L}$ is identically distributed to the output of $L$. This guarantees that $L$ is $(\alpha=1/4,\beta=1/4)$-accurate. 

    However, we have shown in Corollary~\ref{cor:hardness} that such an $L$ cannot exist. So we conclude the non-existence of an $(\alpha=1/4,\beta=1/4)$-accurate and $(\eps = 1, \delta = 1/4n)$-differentially private PAC learner for $\EncThr$.
\end{proof}

\section{Constructing $\FRE$ with $\tfld$ Evaluation} \label{sec:crypto}

We now describe sufficient cryptographic and complexity theoretic assumptions to construct function revealing encryption with any leakage computable by poly-size circuits, including $\tfld$.

\begin{figure}
    \centering
    
    \tikzstyle{block} = [rectangle, draw, text width=6em, text centered, rounded corners, minimum height=6em]
    \tikzstyle{line} = [draw, -latex']
    
    \begin{tikzpicture}[node distance=1cm, auto]
        \node (init) {};
        \node [block] (A) {Single input functional encryption ($\fe$)};
        \node [block, below=3cm of A] (B) {Perfectly correct $\fe$};
        \node [block, right=3cm of A] (C) {Perfectly correct $\tfe$};
        \node[block,right=3cm of B] (D) {Perfectly correct $\FRE$};
        \node[block,right=3cm of C] (E) {Perfectly and strongly correct $\FRE$};

        \path [line] (A) -- node [text width=1.5cm,left,midway,align=center] {\cite{BitanskyV22}} (B);
        
        \path [line] (B) -- node [text width=1.5cm,pos=0.75, below left,align=center ] {\cite{BrakerskiKS18}\\} (C);
        
        \path [line] (C) -- node [text width=1.5cm,midway,right,align=center ]{}(D);
        
         \path [line] (D) -- node [text width=1.5cm,pos=0.20,above right,align=center ] {\cite{BunZ16}} (E);
    \end{tikzpicture}
    \caption{Sketch of construction}
    \label{fig:enter-label}
\end{figure}
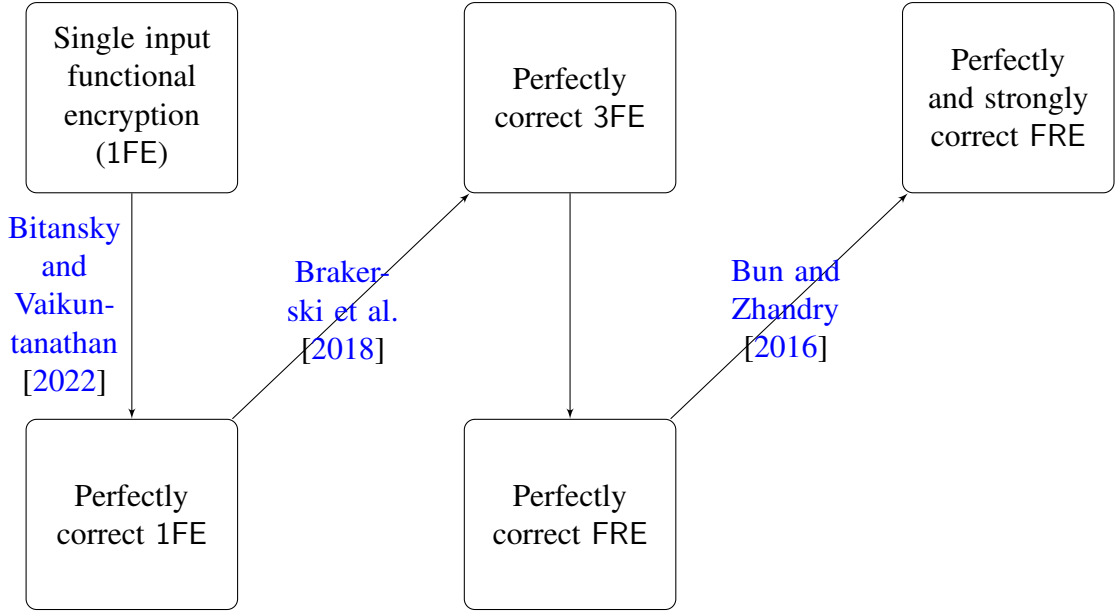

The ``heavy hammer'' in our construction is single-input functional encryption for all poly-size circuits. The existence of this primitive is roughly equivalent to indistinguishability obfuscation; a recent breakthrough of~\cite{JainLS21} showed that both can be based on a slate of reasonable assumptions described below.

\begin{theorem}[\cite{JainLS21}] \label{thm:FE-from-reasonable}
 Let $\lambda$ be a security parameter, $p$ be an efficiently sampleable $\lambda$-bit prime, and $k=k(\lambda)$ be a large enough polynomial. Assume:
    \begin{itemize}
        \item The $\algstyle{SXDH}$ assumption with respect to a bilinear groups of order $p$,
        \item The $\algstyle{LWE}$ assumption with modulus-to-noise ratio $2^{k^\epsilon}$ where $k=k(\lambda)$ is the dimension of the secret,
        \item The existence of $\gamma-$ secure perturbation resilient generators $\Delta\algstyle{RG}\in (\algstyle{deg}\;2,\algstyle{deg}\;d)$ over $\mathbb{Z}_p$ for some constant $d\in \N$, with polynomial stretch.       
    \end{itemize}
    Then there exists a secret-key functional encryption scheme for polynomial sized circuits having adaptive collusion resistant security, full compactness and perfect correctness.   
\end{theorem}

Figure~\ref{fig:enter-label} shows our path for building the perfectly and strongly correct $\FRE$ we need from single-input functional encryption. 

First, we use the following result of~\cite{BitanskyV22} which gives a complexity-theoretic assumption under which we can guarantee correctness with probability 1.

\begin{theorem}[\cite{BitanskyV22}] \label{thm:bv} Assume the existence of one-way functions and functions with deterministic (uniform) time complexity $2^{O(n)}$, but non-deterministic circuit complexity $2^{\Omega(n)}$. Then any cryptographic scheme that is secure under parallel repetitions can be made perfectly correct.
\end{theorem}

Next, we apply a transformation of~\cite{BrakerskiKS18}, who show how to construct a multi-input functional encryption scheme from a single-input functional encryption scheme for all circuits. Note that perfect correctness of the single-input scheme translates into perfect correctness of the resulting multi-input scheme.

\begin{theorem} [\cite{BrakerskiKS18}] \label{thm:perf-corr}
Assume the existence of
\begin{itemize}
\item A private-key single-input functional encryption scheme for all polynomial-size circuits.
\item A pseudorandom function family.
\end{itemize}
Then there exists a private-key three-input functional encryption scheme for all polynomial-size circuits.
\end{theorem}

A function revealing encryption scheme is a special case of a multi-input functional encryption scheme where only a single fixed functionality is supported. So what remains is to ensure ``strong'' correctness. To obtain this, we can invoke a transformation of~\cite{BunZ16}, who showed how to obtain strong correctness for ORE by attaching a NIZK proof that encryption was performed correctly. Their construction (stated as Theorem 4.1 in their paper) is not specific to ORE and holds for general leakage as stated below.

\begin{theorem}[\cite{BunZ16}]
Assuming the existence of a function-revealing encryption scheme with leakage $\lkg$, a perfectly binding commitment scheme, and perfectly sound non-interactive zero knowledge proofs for $\mathsf{NP}$, there is a strongly correct function-revealing encryption scheme with leakage $\lkg$.
\end{theorem}

Perfectly binding commitments can be built from injective one-way functions; moreover, the injectivity requirement can be removed if the circuit lower bound described in Theorem~\ref{thm:perf-corr} holds~\citep{BarakOV07}. Perfectly sound NIZKs can be built from bilinear maps~\citep{GrothOS06}.

Invoking Theorem~\ref{thm:sos} with the construction we've outlined, and using the fact that (functional) encryption implies the existence of one-way functions, we obtain the following separation.

\begin{theorem}\label{thm:main} Assume the existence of functional encryption for poly-size circuits (obtainable via the assumptions in Theorem~\ref{thm:FE-from-reasonable}), functions computable in time $2^{O(n)}$ with non-deterministic circuit complexity $2^{\Omega(n)}$, and perfectly sound non-interactive zero knowledge proofs for $\mathsf{NP}$.
    Then there exists a concept class that is is online learnable in polynomial time with a polynomial mistake bound. However, this class cannot be learned by a $(\alpha=1/4,\beta=1/4)$-accurate and $(\eps = 1, \delta = 1/4n)$-differentially private algorithm.
\end{theorem}

\section{Conclusion}
We conclude with the following open questions.
\begin{itemize}
\item Can we build $\FRE$ schemes satisfying our bisection and log-invariance properties from weaker assumptions? A beautiful line of work~\cite{ChenetteLWW16, CashLOZ16} constructs ``leaky order-revealing encryption schemes'' that enable tantalizingly close functionalities to our $\tfld$. These constructions require much weaker cryptographic assumptions, e,g., just one-way functions and pairings, than what seem to be needed to get multi-input functional encryption for all circuits. 
\item Can one identify a rich, important class of efficient online learners that \emph{can} be efficiently transformed into private PAC learners? 
\item Putting computational complexity aside, can we obtain an improved separation between private sample complexity and non-private sample complexity (characterized by VC dimension) of learning? The current best separation is still only a factor of $\log^*|\mathcal{F}|$~\citep{AlonBLMM22}. Similarly, can we improve our general understanding of the sample complexity of private learning?
\end{itemize}

\section*{Acknowledgments}
MB was supported by NSF CNS-2046425 and a Sloan Research Fellowship, and thanks Mark Zhandry for helpful conversations about order revealing encryption and its variants. AC was supported by NSF CNS-1915763. RD was supported by NSF CNS-2046425.

\bibliographystyle{plainnat}
\bibliography{biblio}

\begin{thebibliography}{45}
\providecommand{\natexlab}[1]{#1}
\providecommand{\url}[1]{\texttt{#1}}
\expandafter\ifx\csname urlstyle\endcsname\relax
  \providecommand{\doi}[1]{doi: #1}\else
  \providecommand{\doi}{doi: \begingroup \urlstyle{rm}\Url}\fi

\bibitem[Aaronson and Rothblum(2019)]{AaronsonR19}
Scott Aaronson and Guy~N. Rothblum.
\newblock Gentle measurement of quantum states and differential privacy.
\newblock In Moses Charikar and Edith Cohen, editors, \emph{Proceedings of the 51st Annual {ACM} {SIGACT} Symposium on Theory of Computing, {STOC} 2019, Phoenix, AZ, USA, June 23-26, 2019}, pages 322--333. {ACM}, 2019.
\newblock \doi{10.1145/3313276.3316378}.
\newblock URL \url{https://doi.org/10.1145/3313276.3316378}.

\bibitem[Agarwal and Singh(2017)]{AgarwalS17}
Naman Agarwal and Karan Singh.
\newblock The price of differential privacy for online learning.
\newblock In Doina Precup and Yee~Whye Teh, editors, \emph{Proceedings of the 34th International Conference on Machine Learning, {ICML} 2017, Sydney, NSW, Australia, 6-11 August 2017}, volume~70 of \emph{Proceedings of Machine Learning Research}, pages 32--40. {PMLR}, 2017.
\newblock URL \url{http://proceedings.mlr.press/v70/agarwal17a.html}.

\bibitem[Alon et~al.(2019)Alon, Livni, Malliaris, and Moran]{AlonLMM19}
Noga Alon, Roi Livni, Maryanthe Malliaris, and Shay Moran.
\newblock Private {PAC} learning implies finite {L}ittlestone dimension.
\newblock In \emph{Proceedings of the 51st Annual ACM Symposium on the Theory of Computing}, STOC '19, New York, NY, USA, 2019. ACM.

\bibitem[Alon et~al.(2022)Alon, Bun, Livni, Malliaris, and Moran]{AlonBLMM22}
Noga Alon, Mark Bun, Roi Livni, Maryanthe Malliaris, and Shay Moran.
\newblock Private and online learnability are equivalent.
\newblock \emph{J. {ACM}}, 69\penalty0 (4):\penalty0 28:1--28:34, 2022.
\newblock \doi{10.1145/3526074}.
\newblock URL \url{https://doi.org/10.1145/3526074}.

\bibitem[Angluin(1988)]{Angluin88}
Dana Angluin.
\newblock Queries and concept learning.
\newblock \emph{Machine Learning}, 2\penalty0 (4):\penalty0 319--342, 1988.

\bibitem[Asi et~al.(2023)Asi, Feldman, Koren, and Talwar]{AsiFKT23}
Hilal Asi, Vitaly Feldman, Tomer Koren, and Kunal Talwar.
\newblock Private online prediction from experts: Separations and faster rates.
\newblock In Gergely Neu and Lorenzo Rosasco, editors, \emph{The Thirty Sixth Annual Conference on Learning Theory, {COLT} 2023, 12-15 July 2023, Bangalore, India}, volume 195 of \emph{Proceedings of Machine Learning Research}, pages 674--699. {PMLR}, 2023.
\newblock URL \url{https://proceedings.mlr.press/v195/asi23a.html}.

\bibitem[Balle et~al.(2018)Balle, Barthe, and Gaboardi]{balle2018privacy}
Borja Balle, Gilles Barthe, and Marco Gaboardi.
\newblock Privacy amplification by subsampling: Tight analyses via couplings and divergences, 2018.

\bibitem[Barak et~al.(2007)Barak, Ong, and Vadhan]{BarakOV07}
Boaz Barak, Shien~Jin Ong, and Salil Vadhan.
\newblock Derandomization in cryptography.
\newblock \emph{SIAM Journal on Computing}, 37\penalty0 (2):\penalty0 380--400, 2007.
\newblock \doi{10.1137/050641958}.

\bibitem[Beimel et~al.(2014)Beimel, Brenner, Kasiviswanathan, and Nissim]{BeimelBKN14}
Amos Beimel, Hai Brenner, Shiva~Prasad Kasiviswanathan, and Kobbi Nissim.
\newblock Bounds on the sample complexity for private learning and private data release.
\newblock \emph{Machine Learning}, 94\penalty0 (3):\penalty0 401--437, 2014.

\bibitem[Beimel et~al.(2016)Beimel, Nissim, and Stemmer]{BeimelNS16}
Amos Beimel, Kobbi Nissim, and Uri Stemmer.
\newblock Private learning and sanitization: Pure vs. approximate differential privacy.
\newblock \emph{Theory of Computing}, 12\penalty0 (1):\penalty0 1--61, 2016.

\bibitem[Beimel et~al.(2018)Beimel, Haitner, Makriyannis, and Omri]{BeimelHMO18}
Amos Beimel, Iftach Haitner, Nikolaos Makriyannis, and Eran Omri.
\newblock Tighter bounds on multi-party coin flipping via augmented weak martingales and differentially private sampling.
\newblock In Mikkel Thorup, editor, \emph{59th {IEEE} Annual Symposium on Foundations of Computer Science, {FOCS} 2018, Paris, France, October 7-9, 2018}, pages 838--849. {IEEE} Computer Society, 2018.
\newblock \doi{10.1109/FOCS.2018.00084}.
\newblock URL \url{https://doi.org/10.1109/FOCS.2018.00084}.

\bibitem[Beimel et~al.(2019)Beimel, Nissim, and Stemmer]{Beimel19Pure}
Amos Beimel, Kobbi Nissim, and Uri Stemmer.
\newblock Characterizing the sample complexity of pure private learners.
\newblock \emph{Journal of Machine Learning Research}, 20\penalty0 (146):\penalty0 1--33, 2019.
\newblock URL \url{http://jmlr.org/papers/v20/18-269.html}.

\bibitem[Bitansky and Vaikuntanathan(2022)]{BitanskyV22}
Nir Bitansky and Vinod Vaikuntanathan.
\newblock A note on perfect correctness by derandomization.
\newblock \emph{J. Cryptol.}, 35\penalty0 (3):\penalty0 18, 2022.
\newblock \doi{10.1007/s00145-022-09428-0}.
\newblock URL \url{https://doi.org/10.1007/s00145-022-09428-0}.

\bibitem[Blum et~al.(2005)Blum, Dwork, McSherry, and Nissim]{BlumDMN05}
Avrim Blum, Cynthia Dwork, Frank McSherry, and Kobbi Nissim.
\newblock Practical privacy: The {SuLQ} framework.
\newblock In \emph{Proceedings of the 24th ACM SIGMOD-SIGACT-SIGART Symposium on Principles of Database Systems}, PODS '05, pages 128--138, New York, NY, USA, 2005. ACM.

\bibitem[Brakerski et~al.(2018)Brakerski, Komargodski, and Segev]{BrakerskiKS18}
Zvika Brakerski, Ilan Komargodski, and Gil Segev.
\newblock Multi-input functional encryption in the private-key setting: Stronger security from weaker assumptions.
\newblock \emph{J. Cryptol.}, 31\penalty0 (2):\penalty0 434--520, 2018.
\newblock \doi{10.1007/s00145-017-9261-0}.
\newblock URL \url{https://doi.org/10.1007/s00145-017-9261-0}.

\bibitem[Bun(2020)]{Bun20}
Mark Bun.
\newblock A computational separation between private learning and online learning.
\newblock In Hugo Larochelle, Marc'Aurelio Ranzato, Raia Hadsell, Maria{-}Florina Balcan, and Hsuan{-}Tien Lin, editors, \emph{Advances in Neural Information Processing Systems 33: Annual Conference on Neural Information Processing Systems 2020, NeurIPS 2020, December 6-12, 2020, virtual}, 2020.
\newblock URL \url{https://proceedings.neurips.cc/paper/2020/hash/ee715daa76f1b51d80343f45547be570-Abstract.html}.

\bibitem[Bun and Zhandry(2016)]{BunZ16}
Mark Bun and Mark Zhandry.
\newblock Order-revealing encryption and the hardness of private learning.
\newblock In Eyal Kushilevitz and Tal Malkin, editors, \emph{Theory of Cryptography - 13th International Conference, {TCC} 2016-A, Tel Aviv, Israel, January 10-13, 2016, Proceedings, Part {I}}, volume 9562 of \emph{Lecture Notes in Computer Science}, pages 176--206. Springer, 2016.

\bibitem[Bun et~al.(2015)Bun, Nissim, Stemmer, and Vadhan]{BunNSV15}
Mark Bun, Kobbi Nissim, Uri Stemmer, and Salil Vadhan.
\newblock Differentially private release and learning of threshold functions.
\newblock In \emph{Proceedings of the 56th Annual IEEE Symposium on Foundations of Computer Science}, FOCS '15, pages 634--649, Washington, DC, USA, 2015. IEEE Computer Society.

\bibitem[Bun et~al.(2023)Bun, Gaboardi, Hopkins, Impagliazzo, Lei, Pitassi, Sivakumar, and Sorrell]{BunGHILPSS23}
Mark Bun, Marco Gaboardi, Max Hopkins, Russell Impagliazzo, Rex Lei, Toniann Pitassi, Satchit Sivakumar, and Jessica Sorrell.
\newblock Stability is stable: Connections between replicability, privacy, and adaptive generalization.
\newblock In Barna Saha and Rocco~A. Servedio, editors, \emph{Proceedings of the 55th Annual {ACM} Symposium on Theory of Computing, {STOC} 2023, Orlando, FL, USA, June 20-23, 2023}, pages 520--527. {ACM}, 2023.
\newblock \doi{10.1145/3564246.3585246}.
\newblock URL \url{https://doi.org/10.1145/3564246.3585246}.

\bibitem[Cash et~al.(2016)Cash, Liu, O'Neill, and Zhang]{CashLOZ16}
David Cash, Feng{-}Hao Liu, Adam O'Neill, and Cong Zhang.
\newblock Reducing the leakage in practical order-revealing encryption.
\newblock \emph{{IACR} Cryptol. ePrint Arch.}, page 661, 2016.
\newblock URL \url{http://eprint.iacr.org/2016/661}.

\bibitem[Chenette et~al.(2016)Chenette, Lewi, Weis, and Wu]{ChenetteLWW16}
Nathan Chenette, Kevin Lewi, Stephen~A. Weis, and David~J. Wu.
\newblock Practical order-revealing encryption with limited leakage.
\newblock In Thomas Peyrin, editor, \emph{Fast Software Encryption - 23rd International Conference, {FSE} 2016, Bochum, Germany, March 20-23, 2016, Revised Selected Papers}, volume 9783 of \emph{Lecture Notes in Computer Science}, pages 474--493. Springer, 2016.
\newblock \doi{10.1007/978-3-662-52993-5\_24}.
\newblock URL \url{https://doi.org/10.1007/978-3-662-52993-5\_24}.

\bibitem[Dwork and Lei(2009)]{DworkL09}
Cynthia Dwork and Jing Lei.
\newblock Differential privacy and robust statistics.
\newblock In \emph{Proceedings of the 41st Annual ACM Symposium on the Theory of Computing}, STOC '09, pages 371--380, New York, NY, USA, 2009. ACM.

\bibitem[Dwork et~al.(2006{\natexlab{a}})Dwork, Kenthapadi, McSherry, Mironov, and Naor]{DworkKMMN06}
Cynthia Dwork, Krishnaram Kenthapadi, Frank McSherry, Ilya Mironov, and Moni Naor.
\newblock Our data, ourselves: Privacy via distributed noise generation.
\newblock In \emph{Proceedings of the 24th Annual International Conference on the Theory and Applications of Cryptographic Techniques}, EUROCRYPT '06, pages 486--503, Berlin, Heidelberg, 2006{\natexlab{a}}. Springer.

\bibitem[Dwork et~al.(2006{\natexlab{b}})Dwork, McSherry, Nissim, and Smith]{DworkMNS06}
Cynthia Dwork, Frank McSherry, Kobbi Nissim, and Adam Smith.
\newblock Calibrating noise to sensitivity in private data analysis.
\newblock In \emph{Proceedings of the 3rd Conference on Theory of Cryptography}, TCC '06, pages 265--284, Berlin, Heidelberg, 2006{\natexlab{b}}. Springer.

\bibitem[Dwork et~al.(2015)Dwork, Feldman, Hardt, Pitassi, Reingold, and Roth]{DworkFHPRR15}
Cynthia Dwork, Vitaly Feldman, Moritz Hardt, Toniann Pitassi, Omer Reingold, and Aaron Roth.
\newblock The reusable holdout: Preserving validity in adaptive data analysis.
\newblock \emph{Science}, 349\penalty0 (6248):\penalty0 636--638, 2015.

\bibitem[Feldman and Xiao(2015)]{FeldmanX15}
Vitaly Feldman and David Xiao.
\newblock Sample complexity bounds on differentially private learning via communication complexity.
\newblock \emph{SIAM Journal on Computing}, 44\penalty0 (6):\penalty0 1740--1764, 2015.

\bibitem[Frances and Litman(1998)]{FrancesL98}
Moti Frances and Ami Litman.
\newblock Optimal mistake bound learning is hard.
\newblock \emph{Inf. Comput.}, 144\penalty0 (1):\penalty0 66--82, 1998.
\newblock \doi{10.1006/inco.1998.2709}.
\newblock URL \url{https://doi.org/10.1006/inco.1998.2709}.

\bibitem[Ghazi et~al.(2021)Ghazi, Golowich, Kumar, and Manurangsi]{GhaziGKM21}
Badih Ghazi, Noah Golowich, Ravi Kumar, and Pasin Manurangsi.
\newblock Sample-efficient proper {PAC} learning with approximate differential privacy.
\newblock In Samir Khuller and Virginia~Vassilevska Williams, editors, \emph{{STOC} '21: 53rd Annual {ACM} {SIGACT} Symposium on Theory of Computing, Virtual Event, Italy, June 21-25, 2021}, pages 183--196. {ACM}, 2021.
\newblock \doi{10.1145/3406325.3451028}.
\newblock URL \url{https://doi.org/10.1145/3406325.3451028}.

\bibitem[Gonen et~al.(2019)Gonen, Hazan, and Moran]{GonenHM19}
Alon Gonen, Elad Hazan, and Shay Moran.
\newblock Private learning implies online learning: An efficient reduction.
\newblock \emph{{NeurIPS}}, 2019.

\bibitem[Groth et~al.(2012)Groth, Ostrovsky, and Sahai]{GrothOS06}
Jens Groth, Rafail Ostrovsky, and Amit Sahai.
\newblock New techniques for noninteractive zero-knowledge.
\newblock \emph{J. ACM}, 59\penalty0 (3), jun 2012.
\newblock ISSN 0004-5411.
\newblock \doi{10.1145/2220357.2220358}.
\newblock URL \url{https://doi.org/10.1145/2220357.2220358}.

\bibitem[Hardt and Ullman(2014)]{HardtU14}
Moritz Hardt and Jonathan Ullman.
\newblock Preventing false discovery in interactive data analysis is hard.
\newblock In \emph{Proceedings of the 55th Annual IEEE Symposium on Foundations of Computer Science}, FOCS '14, pages 454--463, Washington, DC, USA, 2014. IEEE Computer Society.

\bibitem[Jain et~al.(2021)Jain, Lin, and Sahai]{JainLS21}
Aayush Jain, Huijia Lin, and Amit Sahai.
\newblock Indistinguishability obfuscation from well-founded assumptions.
\newblock In Samir Khuller and Virginia~Vassilevska Williams, editors, \emph{{STOC} '21: 53rd Annual {ACM} {SIGACT} Symposium on Theory of Computing, Virtual Event, Italy, June 21-25, 2021}, pages 60--73. {ACM}, 2021.
\newblock \doi{10.1145/3406325.3451093}.
\newblock URL \url{https://doi.org/10.1145/3406325.3451093}.

\bibitem[Kaplan et~al.(2019)Kaplan, Ligett, Mansour, Naor, and Stemmer]{KaplanLMNS19}
Haim Kaplan, Katrina Ligett, Yishay Mansour, Moni Naor, and Uri Stemmer.
\newblock Privately learning thresholds: Closing the exponential gap.
\newblock \emph{CoRR}, abs/1911.10137, 2019.
\newblock URL \url{http://arxiv.org/abs/1911.10137}.

\bibitem[Kasiviswanathan et~al.(2011)Kasiviswanathan, Lee, Nissim, Raskhodnikova, and Smith]{KasiviswanathanLNRS11}
Shiva~Prasad Kasiviswanathan, Homin~K. Lee, Kobbi Nissim, Sofya Raskhodnikova, and Adam Smith.
\newblock What can we learn privately?
\newblock \emph{SIAM Journal on Computing}, 40\penalty0 (3):\penalty0 793--826, 2011.

\bibitem[Kearns(1998)]{Kearns98}
Michael~J. Kearns.
\newblock Efficient noise-tolerant learning from statistical queries.
\newblock \emph{J. {ACM}}, 45\penalty0 (6):\penalty0 983--1006, 1998.
\newblock \doi{10.1145/293347.293351}.
\newblock URL \url{https://doi.org/10.1145/293347.293351}.

\bibitem[Kearns et~al.(1987)Kearns, Li, Pitt, and Valiant]{KearnsLPV87}
Michael~J. Kearns, Ming Li, Leonard Pitt, and Leslie~G. Valiant.
\newblock On the learnability of boolean formulae.
\newblock In Alfred~V. Aho, editor, \emph{Proceedings of the 19th Annual {ACM} Symposium on Theory of Computing, 1987, New York, New York, {USA}}, pages 285--295. {ACM}, 1987.
\newblock \doi{10.1145/28395.28426}.
\newblock URL \url{https://doi.org/10.1145/28395.28426}.

\bibitem[Littlestone(1987)]{Littlestone87online}
Nick Littlestone.
\newblock Learning quickly when irrelevant attributes abound: {A} new linear-threshold algorithm.
\newblock \emph{Machine Learning}, 2\penalty0 (4):\penalty0 285--318, 1987.

\bibitem[Littlestone(1989)]{Littlestone89}
Nick Littlestone.
\newblock From on-line to batch learning.
\newblock In Ronald~L. Rivest, David Haussler, and Manfred~K. Warmuth, editors, \emph{Proceedings of the Second Annual Workshop on Computational Learning Theory, {COLT} 1989, Santa Cruz, CA, USA, July 31 - August 2, 1989}, pages 269--284. Morgan Kaufmann, 1989.
\newblock URL \url{http://dl.acm.org/citation.cfm?id=93365}.

\bibitem[Manurangsi(2023)]{Manurangsi23}
Pasin Manurangsi.
\newblock Improved inapproximability of {VC} dimension and littlestone's dimension via (unbalanced) biclique.
\newblock In Yael~Tauman Kalai, editor, \emph{14th Innovations in Theoretical Computer Science Conference, {ITCS} 2023, January 10-13, 2023, MIT, Cambridge, Massachusetts, {USA}}, volume 251 of \emph{LIPIcs}, pages 85:1--85:18. Schloss Dagstuhl - Leibniz-Zentrum f{\"{u}}r Informatik, 2023.
\newblock \doi{10.4230/LIPIcs.ITCS.2023.85}.
\newblock URL \url{https://doi.org/10.4230/LIPIcs.ITCS.2023.85}.

\bibitem[Manurangsi and Rubinstein(2017)]{ManurangsiR17}
Pasin Manurangsi and Aviad Rubinstein.
\newblock Inapproximability of {VC} dimension and littlestone's dimension.
\newblock \emph{CoRR}, abs/1705.09517, 2017.
\newblock URL \url{http://arxiv.org/abs/1705.09517}.

\bibitem[McSherry and Talwar(2007)]{McSherryT07}
Frank McSherry and Kunal Talwar.
\newblock Mechanism design via differential privacy.
\newblock In \emph{Proceedings of the 48th Annual IEEE Symposium on Foundations of Computer Science}, FOCS '07, pages 94--103, Washington, DC, USA, 2007. IEEE Computer Society.

\bibitem[Nissim et~al.(2012)Nissim, Smorodinsky, and Tennenholtz]{NissimST12}
Kobbi Nissim, Rann Smorodinsky, and Moshe Tennenholtz.
\newblock Approximately optimal mechanism design via differential privacy.
\newblock In Shafi Goldwasser, editor, \emph{Innovations in Theoretical Computer Science 2012, Cambridge, MA, USA, January 8-10, 2012}, pages 203--213. {ACM}, 2012.
\newblock \doi{10.1145/2090236.2090254}.
\newblock URL \url{https://doi.org/10.1145/2090236.2090254}.

\bibitem[Sadigurschi and Stemmer(2021)]{SadigurschiS21}
Menachem Sadigurschi and Uri Stemmer.
\newblock On the sample complexity of privately learning axis-aligned rectangles.
\newblock In Marc'Aurelio Ranzato, Alina Beygelzimer, Yann~N. Dauphin, Percy Liang, and Jennifer~Wortman Vaughan, editors, \emph{Advances in Neural Information Processing Systems 34: Annual Conference on Neural Information Processing Systems 2021, NeurIPS 2021, December 6-14, 2021, virtual}, pages 28286--28297, 2021.
\newblock URL \url{https://proceedings.neurips.cc/paper/2021/hash/ee0e95249268b86ff2053bef214bfeda-Abstract.html}.

\bibitem[Schaefer(1999)]{Schaefer99}
Marcus Schaefer.
\newblock Deciding the {V}apnik-{$\text{\v{C}}$}ervonenkis dimension is {$\Sigma^{\text{p}}_3$}-complete.
\newblock \emph{J. Comput. Syst. Sci.}, 58\penalty0 (1):\penalty0 177--182, 1999.
\newblock \doi{10.1006/jcss.1998.1602}.
\newblock URL \url{https://doi.org/10.1006/jcss.1998.1602}.

\bibitem[Valiant(1984)]{Valiant84}
Leslie~G. Valiant.
\newblock A theory of the learnable.
\newblock \emph{Communications of the ACM}, 27\penalty0 (11):\penalty0 1134--1142, 1984.

\end{thebibliography}

\appendix
\end{document}